\documentclass{article}

\usepackage{microtype}
\usepackage{graphicx}
\usepackage{subfigure}
\usepackage{booktabs} % for professional tables

\usepackage{hyperref}

\usepackage[accepted]{icml2019}

\usepackage{amssymb, amsthm}
\usepackage[reqno]{amsmath}

\makeatletter
\newcommand{\leqnomode}{\tagsleft@true}
\newcommand{\reqnomode}{\tagsleft@false}
\makeatother
\usepackage{safecv} % project specific package

\makeatletter
\newtheorem*{rep@theorem}{\rep@title}
\newcommand{\newreptheorem}[2]{%
\newenvironment{rep#1}[1]{%
 \def\rep@title{#2 \ref{##1}}%
 \begin{rep@theorem}}%
 {\end{rep@theorem}}}
\makeatother

\theoremstyle{plain}

\newtheorem{proposition}{Proposition}
\theoremstyle{definition}
\newtheorem{definition}{Definition}
\newtheorem{remark}{Remark}

\newtheorem{corollary}{Corollary}
\newtheorem{lemma}{Lemma}

\newreptheorem{lemma}{Lemma}
\newreptheorem{proposition}{Proposition}
\newreptheorem{theorem}{Theorem}
\newreptheorem{corollary}{Corollary}

\graphicspath{{images/},{prebuiltimages/}}

\usepackage{cleveref} % important for avoiding th/prop/lemma pb when renamming.

\icmltitlerunning{Safe Grid Search with Optimal Complexity}

\begin{document}

\twocolumn[
% \icmltitle{Submission and Formatting Instructions for \\
%            International Conference on Machine Learning (ICML 2019)}
\icmltitle{Safe Grid Search with Optimal Complexity}



\icmlsetsymbol{equal}{*}

% \begin{icmlauthorlist}
% \icmlauthor{Aeiau Zzzz}{equal,to}
% \icmlauthor{Bauiu C.~Yyyy}{equal,to,goo}
% \icmlauthor{Cieua Vvvvv}{goo}
% \icmlauthor{Iaesut Saoeu}{ed}
% \icmlauthor{Fiuea Rrrr}{to}
% \icmlauthor{Tateu H.~Yasehe}{ed,to,goo}
% \icmlauthor{Aaoeu Iasoh}{goo}
% \icmlauthor{Buiui Eueu}{ed}
% \icmlauthor{Aeuia Zzzz}{ed}
% \icmlauthor{Bieea C.~Yyyy}{to,goo}
% \icmlauthor{Teoau Xxxx}{ed}
% \icmlauthor{Eee Pppp}{ed}
% \end{icmlauthorlist}

\begin{icmlauthorlist}
\icmlauthor{Eugene Ndiaye}{riken}
\icmlauthor{Tam Le}{riken}
\icmlauthor{Olivier Fercoq}{telecom}
\icmlauthor{Joseph Salmon}{montp}
\icmlauthor{Ichiro Takeuchi}{nitech}
\end{icmlauthorlist}

\icmlaffiliation{riken}{Riken AIP}
\icmlaffiliation{telecom}{LTCI, T\'el\'ecom ParisTech, Universit\'e Paris-Saclay}
\icmlaffiliation{montp}{IMAG, Univ Montpellier, CNRS, Montpellier, France}
\icmlaffiliation{nitech}{Nagoya Institute of Technology}

\icmlcorrespondingauthor{E. Ndiaye}{eugene.ndiaye@riken.jp}

\icmlkeywords{Machine Learning, ICML}

\vskip 0.3in
]



\printAffiliationsAndNotice{}  % leave blank if no need to mention equal contribution

\begin{abstract}
Popular machine learning estimators involve regularization parameters that can be challenging to tune, and standard strategies rely on grid search for this task.
In this paper, we revisit the techniques of approximating the regularization  path up to predefined tolerance $\epsilon$ in a unified framework and show that its complexity is $O(1/\sqrt[d]{\epsilon})$ for uniformly convex loss of order $d \geq 2$ and $O(1/\sqrt{\epsilon})$ for Generalized Self-Concordant functions.
This framework encompasses least-squares but also logistic regression, a case that as far as we know was not handled as precisely in previous works.
We leverage our technique to provide refined bounds on the validation error as well as a practical algorithm for hyperparameter tuning.
The latter has global convergence guarantee when targeting a prescribed accuracy on the validation set.
Last but not least, our approach helps relieving the practitioner from the (often neglected) task of selecting a stopping criterion when optimizing over the training set: our method automatically calibrates this criterion based on the targeted accuracy on the validation set.
\end{abstract}

%!TEX root = ../icml_supp.tex

%%%%%%%%%%%%%%%%%%%%%%%%%%%%%%%%%%%%%%%%%%%%%%%%%%%%%%%%%%%%%%%%%%%%%%%%%%%%%%%
%%%%%%%%%%%%%%%%%%%%%%%%%%%%%%%%%%%%%%%%%%%%%%%%%%%%%%%%%%%%%%%%%%%%%%%%%%%%%%%
\section{Introduction}
\label{sec:introduction}
%%%%%%%%%%%%%%%%%%%%%%%%%%%%%%%%%%%%%%%%%%%%%%%%%%%%%%%%%%%%%%%%%%%%%%%%%%%%%%%
%%%%%%%%%%%%%%%%%%%%%%%%%%%%%%%%%%%%%%%%%%%%%%%%%%%%%%%%%%%%%%%%%%%%%%%%%%%%%%%
%XXX TODO fo andr revision : we need to add citations to bayesian optimization and other autoML methods...
Various machine learning problems are formulated as minimization of an empirical loss function $f$ plus a regularization function $\Omega$ whose calibration is controlled by a hyperparameter $\lambda$.
The choice of $\lambda$ is crucial in practice since it directly influences the generalization performance of the estimator, \ie its score on unseen data sets.
The most popular method in such a context is cross-validation (or some variant, see \cite{Arlot_Celisse10} for a detailed review).
For simplicity, we investigate here the simplest case, the holdout version. It consists in splitting the data in two parts: on the first part (\emph{training set}) the method is trained for a pre-defined collection of candidates $\Lambda_{T}:=\{\lambda_{0},\dots, \lambda_{T-1}\}$, and on the second part (\emph{validation set}), the best parameter is selected among the $T$ candidates.

For a piecewise quadratic loss $f$ and a piecewise linear regularization $\Omega$ (\eg the Lasso estimator), \citet{Osborne_Presnnell_Turlach00, Rosset_Zhu07} have shown that the set of solutions follows a piecewise linear curve \wrt to the parameter $\lambda$.
There are several algorithms that can generate the full path by maintaining optimality conditions when the regularization parameter varies.
This is what \texttt{LARS} is performing for Lasso \citep{Efron_Hastie_Johnstone_Tibshirani04}, but similar approaches exist for SVM \citep{Hastie_Rosset_Tibshirani_Zhu04} or generalized linear models (GLM) \citep{Park_Hastie07}.
Unfortunately, these methods have some drawbacks that can be critical in many situations:

% \begin{itemize}
$~\bullet$ their worst case complexity, \ie the number of linear segments, is exponential in the dimension $p$ of the problem \citep{Gartner_Jaggi_Maria12} leading to unpractical algorithms. Recently, \citet{Li_Singer18} have shown that for some specific design matrix with $n$ observations, a polynomial complexity of $O(n \times p^6)$ can be obtained. Note that even in a more favorable cases of linear complexity in $p$, the exact path can be expensive to compute when the dimension $p$ is large.

$~\bullet$ they suffer from numerical instabilities due to multiple and expensive inversion of ill-conditioned matrix. As a result, these algorithms may fail before exploring the entire path, a common issue for small regularization parameter.

$~\bullet$ they lack flexibility when it comes at incorporating different statistical learning tasks because they usually rely on specific algebra to handle the structure of the regularization and loss functions.
As far as we know, they can be applied only to a limited number of cases and we are not aware of a general framework that bypasses these issues.

$~\bullet$ they cannot benefit of early stopping. Following \citet{Bottou_Bousquet08}, it is not necessary to optimize below the statistical error for suitable generalization.
Exact regularization path algorithms need to maintain optimality conditions as the hyperparameter varies, which is time consuming.
% \end{itemize}
%

To overcome these issues, an $\epsilon$-approximation of the solution path was proposed and optimal complexity was proven to be $O(1/\epsilon)$ by \cite{Giesen_Jaggi_Laue10} in a fairly general setting. Then, \citet{Mairal_Yu12} provided an interesting algorithm whose complexity is $O(1/\sqrt{\epsilon})$ for the Lasso case. The latter result was then extended by \cite{Giesen_Laue_Mueller_Swiercy12} to objective function with quadratic lower bound while providing a lower and upper bound of order $O(1/\sqrt{\epsilon})$. Unfortunately, these assumptions fail to hold for many problems, including logistic regression or Huber loss.

Following such ideas, \cite{Shibagaki_Suzuki_Karasuyama_Takeuchi15} have proposed, for classification problems, to approximate the regularization path on the hold-out cross-validation error.
Indeed, the latter is a more natural criterion to monitor when one aims at selecting a hyperparameter guaranteed to achieve the best validation error. The main idea is to construct upper and lower bounds of the validation error as simple functions of the regularization parameter.
Hence by sequentially varying the parameters, one can estimate a range of parameter for which the validation error gap (\ie the difference with the validation error achieved by the best parameter) is smaller than an accuracy $\epsilon_v>0$.

\paragraph{Contributions.} We revisit the approximation of the solution and validation path in a unified framework under general regularity assumptions commonly met in machine learning.
We encompass both classification and regression problems and provide a complexity analysis along with explicit optimality guarantees.
We highlight the relationship between the regularity of the loss function and the complexity of the approximation path.
We prove that its complexity is $O(1/\sqrt[d]{\epsilon})$ for uniformly convex loss of order $d \geq 2$ (see \citet[Definition 10.5]{Bauschke_Combettes11}) and $O(1/\sqrt{\epsilon})$ for the logistic loss thanks to a refined measure of its curvature throughout its Generalized Self-Concordant properties \citep{Sun_Tran-Dinh17}.
As far as we know, the previously known approximation path algorithms cannot handle these cases.
We provide an algorithm with global convergence property for selecting a hyperparameter with a validation error $\epsilon_v$-close to the optimal hyperparameter from a given grid.
We bring a natural stopping criterion when optimizing over the training set making this criterion automatically calibrated.

Our implementation is available at \url{https://github.com/EugeneNdiaye/safe_grid_search}.

\begin{table}[t]
\centering

{\small
\begin{tabular}{|c|c|c|}
\hline
 & Lasso & Logistic regr. \\
\hline
$f_i(z)$& ${(y_i-z)^2}/{2}$ & $ \log(1+\mathrm{e}^z) -y_i z$ \\
\hline
$f_i^*(u)$& $({(u-y_i)^2-y_i^2})/{2}$ & $\Nh(u+y_i)$  \\
\hline
$\mathcal{V}_{f^*, x}(u)$ & $\normin{u}_{2}^{2}/2$ & $w_{4}({\normin{u}_{x}^{2}}/{\normin{u}_{2}}) \normin{u}_{u}^{2}$ \\
\hline
\end{tabular}}
\caption{
$w_{4}(\tau) =\frac{(1-\tau)\log(1-\tau) + \tau}{\tau^2}$}
and
$\mathrm{Nh}(x) = x\log(x) + (1-x)\log(1-x)$
\label{tab:summary}
\end{table}
%!TEX root = ../icml_supp.tex

\paragraph{Notation.}

Given a proper, closed and convex function $f: \bbR^n \to \bbR \cup \{+\infty\}$, we denote $\dom f = \{x \in \bbR^n: f(x) < +\infty\}$. If $f$ is a twice continuously differentiable function with positive definite Hessian $\nabla^2 f(x)$ at any $x\in\dom f$, we denote $\norm{z}_x = \sqrt{\langle \nabla^2 f(x)z, z\rangle}$. The Fenchel-Legendre transform of $f$ is the function $f^*:\bbR^n \to \bbR \cup \{+\infty\}$ defined by $f^*(x^*) = \sup_{x \in \dom f} \langle x^* , x \rangle - f(x)$.
% and its sub-differential at $x$ is $\partial f (x) = \{ z \in \bbR^n: \forall y, f(y) \geq f(x) + \langle z, y - x\rangle \}$. The dual norm of a norm $\norm{\cdot}$ over $\bbR^n$ is $\norm{u}_*=\sup_{\norm{z} \leq 1} \langle z , u \rangle$.
The support function of a nonempty set $C$ is defined as $\sigma_C (x) = \sup_{c \in C} \langle c,x \rangle$. If $C$ is closed, convex and contains $0$, we define its polar as $\sigma_{C}^{\circ}(x^*) = \sup_{\sigma_C (x) \leq 1} \langle x^*, x \rangle$. We denote by $[T]$ the set $\{1, \ldots, T\}$ for any non zero integer $T$. The vector of observations is $y \in \bbR^n$ and the design matrix $X= [x_1,\dots,x_n]^\top \in \bbR^{n\times p}$ has $n$ observations row-wise, and $p$ features (column-wise).

% XXX TODO for revision: clarify the impact / usage of lambda_max as in lasao and other sparse cases \jrs{mention here something about $\lambda_min, \lambda_max$, \eg for Lasso or other stuff. and never come back to it.}

%%%%%%%%%%%%%%%%%%%%%%%%%%%%%%%%%%%%%%%%%%%%%%%%%%%%%%%%%%%%%%%%%%%%%%%%%%%%%%%
\section{Problem setup}
\label{sec:setup}
%%%%%%%%%%%%%%%%%%%%%%%%%%%%%%%%%%%%%%%%%%%%%%%%%%%%%%%%%%%%%%%%%%%%%%%%%%%%%%%

Let us consider the class of regularized learning methods expressed as convex optimization problems, such as (regularized) GLM \citep{McCullagh_Nelder89}:
\begin{equation}
\tbeta{\lambda} \in \argmin_{\beta \in \bbR^p} \underbrace{f(X\beta) + \lambda \Omega(\beta)}_{P_{\lambda}(\beta)} \quad \mathrm{(Primal)}. \label{eq:primal_problem}
\end{equation}
We highlight two important cases: the regularized least-squares and logistic regression where the loss functions are written as an empirical risk $f(X\beta)=\sum_{i\in[n]} f_i (x_{i}^{\top}\beta)$ with the $f_i$'s given in \Cref{tab:summary}. The penalty term is often used to incorporate prior knowledges by enforcing a certain regularity on the solutions.
For instance, choosing a Ridge penalty \citep{Hoerl_Kennard70} $\Omega(\cdot) = \tfrac{1}{2} \norm{\cdot}_{2}^{2}$ improves the stability of the resolution of inverse problems while $\Omega(\cdot) = \norm{\cdot}_1$ imposes sparsity at the feature level, a motivation that led to the Lasso estimator \citep{Tibshirani96}; see also \cite{Bach_Jenatton_Mairal_Obozinski12} for extensions to other structured penalties.

\begin{figure}[!t]
\includegraphics[width=\columnwidth]{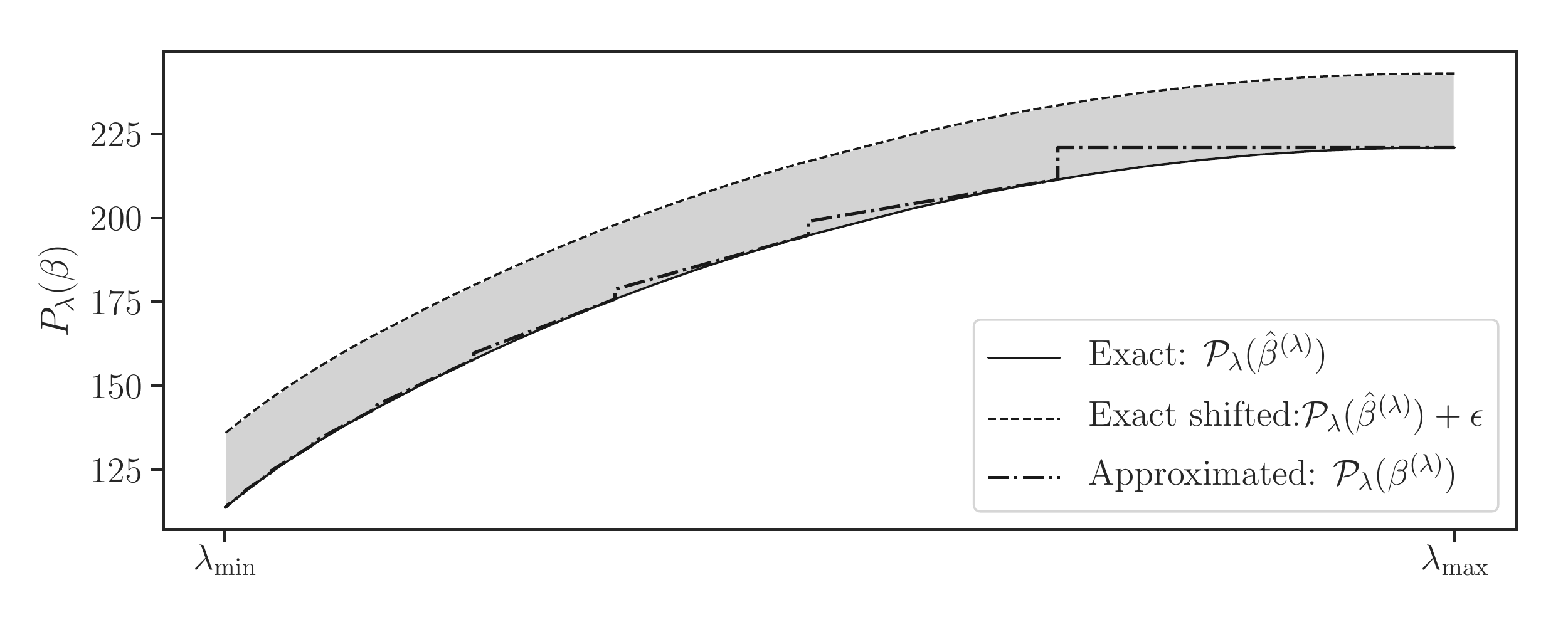}
\vspace{-0.7cm}
\caption{Illustration of the approximation path for the Lasso at accuracy $\epsilon=\norm{y}_{2}^{2}/20$. We choose $\lambda_{\max} = \normin{X^\top y}_{\infty}$ and $\lambda_{\min}=\lambda_{\max}/50$.
The shaded gray region shows the interval where any $\epsilon$-path must lie.
The exact path is computed with the \texttt{LassoLars} on \texttt{diabetes} data from \texttt{sklearn}.\label{fig:approx_lars_diabetes}}
\end{figure}

In practice, obtaining $\tbeta{\lambda}$, an exact solution to Problem~\eqref{eq:primal_problem} is unpractical and one aims achieving a prescribed precision $\epsilon > 0$.
More precisely, a (primal) vector $\beta^{(\lambda)}:=\beta^{(\lambda,\epsilon)}$ (we will drop the dependency in $\epsilon$ for readability) is referred to as an $\epsilon$-solution for $\lambda$ if its (primal) objective value is optimal at precision $\epsilon$:
\begin{equation}\label{eq:epsilon_solution}
P_{\lambda}(\beta^{(\lambda)}) - P_{\lambda}(\tbeta{\lambda}) \leq \epsilon \enspace.
\end{equation}
We recall and illustrate the notion of approximation path in \Cref{fig:approx_lars_diabetes} as described by \citet{Giesen_Laue_Mueller_Swiercy12}.

\begin{definition}[$\epsilon$-path] \label{def:eps_path}
A set $\mathcal{P_{\epsilon}} \subset \bbR^p$ is called an $\epsilon$-path for a parameter range $[\lambdamin, \lambdamax]$ if
\begin{equation}
\forall \lambda \in [\lambdamin, \lambdamax], \exists \text{ an } \epsilon\text{-solution } \beta^{(\lambda)} \in \mathcal{P_{\epsilon}} \enspace.
\end{equation}
We call \emph{path complexity} $T_{\epsilon}$
% for Problem~\eqref{eq:primal_problem} 
the cardinality of the $\epsilon$-path.
\end{definition}

To achieve the targeted $\epsilon$-precision in~\eqref{eq:epsilon_solution} over a whole path and construct an $\epsilon$-path \footnote{note that such a path depends on exact solutions $\tbeta{\lambda}$'s}, we rely on \emph{duality gap} evaluations.
For that, we compute $\epsilon_{c}$-solutions\footnote{the $c$ stands for computational in $\epsilon_{c}$} (for an accuracy $\epsilon_{c}<\epsilon$) over a finite grid, and then we control the gap variations w.r.t. $\lambda$ to achieve the prescribed $\epsilon$-precision over the whole range $[\lambdamin, \lambdamax]$; see \Cref{alg:adaptive_training_path}.
We now recall the Fenchel duality \citep[Chapter 31]{Rockafellar97}:
\begin{align}
\ttheta{\lambda} &\in \argmax_{\theta \in \bbR^n} \underbrace{-f^*(-\lambda \theta) - \lambda \Omega^{*}(X^\top\theta)}_{D_{\lambda}(\theta)} \quad \mathrm{(Dual)}. \label{eq:dual_problem}
\end{align}
For a primal/dual pair $(\beta, \theta) \in \dom P_{\lambda} \times \dom D_{\lambda} $, the duality gap is the difference between primal and dual objectives:
\begin{align*}
\Gap_{\lambda}(\beta, \theta) %&= P_{\lambda}(\beta) - D_{\lambda}(\theta)\\
&= f(X\beta) + f^*(-\lambda \theta) + \lambda (\Omega(\beta) + \Omega^{*}(X^\top \theta)) \enspace.
\end{align*}
Weak duality yields $D_{\lambda}(\theta) \leq P_{\lambda}(\beta)$ and
\begin{equation}\label{eq:dual_gap_certificate}
P_{\lambda}(\beta) - P_{\lambda}(\tbeta{\lambda}) \leq \Gap_{\lambda}(\beta, \theta)\enspace,
\end{equation}
explaining the interest of the duality gap as an optimality certificate.
Using \eqref{eq:dual_gap_certificate}, we can safely construct an approximation path for Problem~\eqref{eq:primal_problem} : if $\beta^{(\lambda)}$ is an $\epsilon$-solution for $\lambda$, it is guaranteed to remain one for all parameters $\lambda'$ such that $\Gap_{\lambda'}(\beta^{(\lambda)}, \theta^{(\lambda)}) \leq \epsilon$.
Since the function $\lambda' \mapsto \Gap_{\lambda'}(\beta^{(\lambda)}, \theta^{(\lambda)})$ does not exhibit a simple dependence in $\lambda$, we rely on an upper bound on the gap encoding the structural regularity of the loss function (\eg 1-dimensional quadratics for strongly convex functions).
This bound controls the optimization error as $\lambda$ varies while preserving optimal complexity on the approximation path.

\begin{figure*}[!ht]
\centering
\subfigure[Uniform unilateral approximation path (as in \Cref{prop:unilateral})]{\includegraphics[width=\columnwidth]{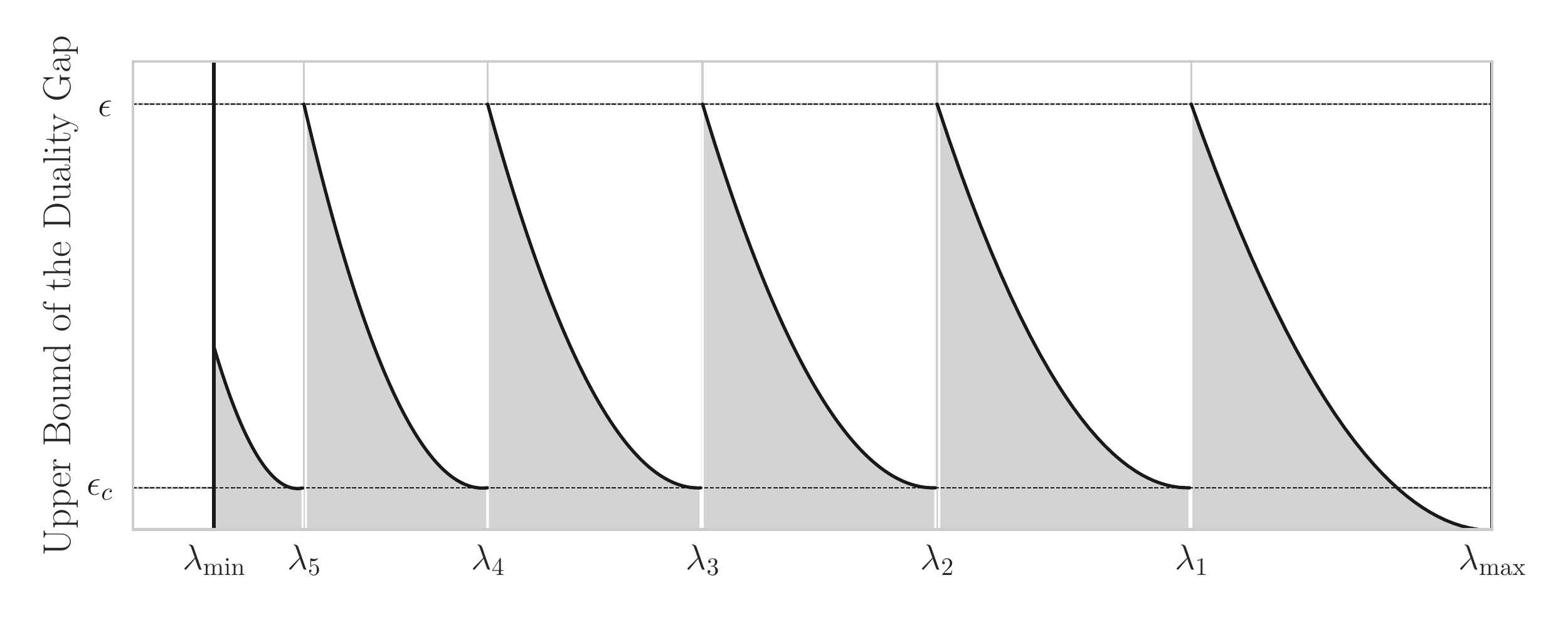}}
\label{subfig:approximation_gap_lasso_unit}
\hspace{0.25cm}
\subfigure[Bilateral approximation path (as in \Cref{prop:bilateral})]{\includegraphics[width=\columnwidth]{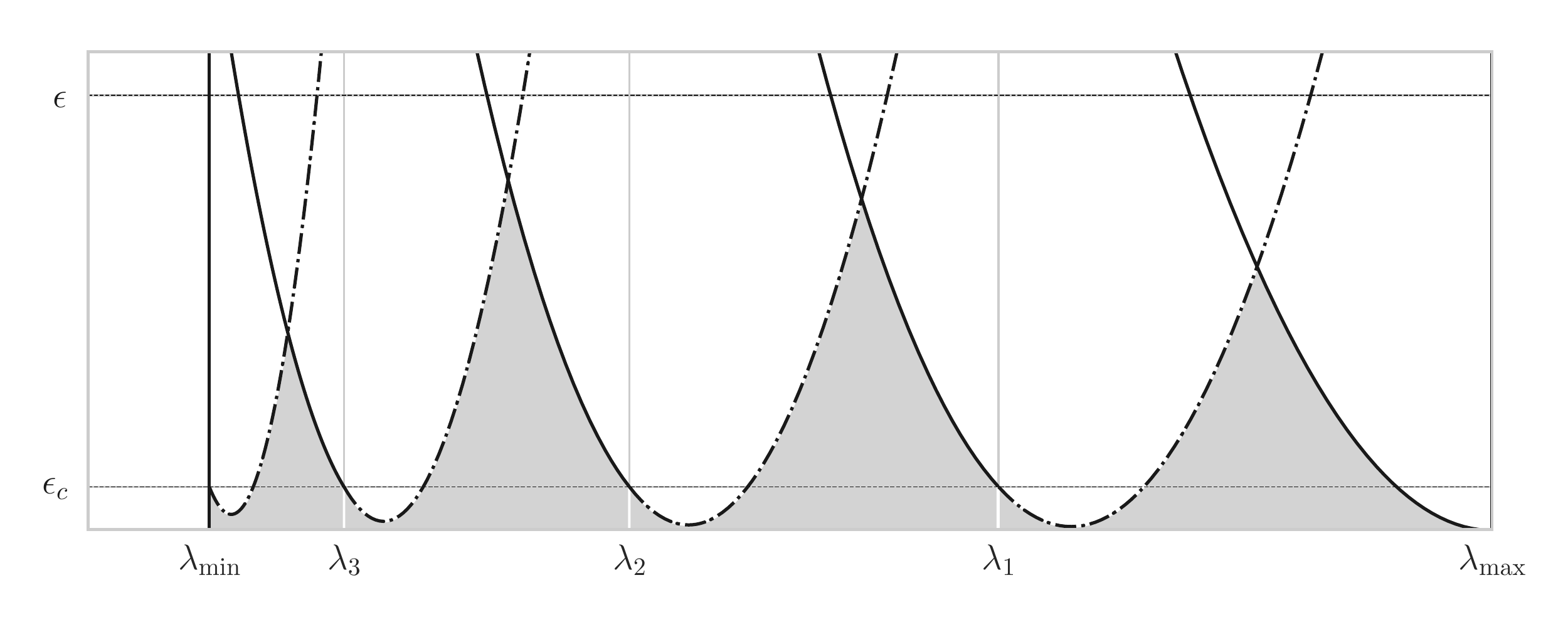}}
\label{subfig:approximation_gap_lasso_bilat}
\caption{Illustration of the construction of $\epsilon$-paths for the Lasso on synthetic dataset generated with $\texttt{sklearn}$ as $X, y =$ \texttt{make\_regression}$(n=30, p=150)$ at accuracy $\epsilon=\norm{y}_{2}^{2}/40$ and $\epsilon_{c}= \epsilon/10$. We choose $\lambda_{\max} = \normin{X^\top y}_{\infty}$ and $\lambda_{\min}=\lambda_{\max}/20$.
For Lasso the bounds are piece-wise quadratic.
The shaded gray regions correspond to the regions where the true value of the duality gap lies. We obtain a path complexity of $T_{\epsilon}=6$ (resp. $T_{\epsilon}=4$) for the unilateral (resp. bilateral) path over $[\lambda_{\min},\lambda_{\max}]$.}
\label{fig:approximation_gap_lasso}
\end{figure*}

%%%%%%%%%%%%%%%%%%%%%%%%%%%%%%%%%%%%%%%%%%%%%%%%%%%%%%%%%%%%%%%%%%%%%%%%%%%%%%%
\section{Bounds and approximation path}
\label{sec:Approximation_Path}
%%%%%%%%%%%%%%%%%%%%%%%%%%%%%%%%%%%%%%%%%%%%%%%%%%%%%%%%%%%%%%%%%%%%%%%%%%%%%%%

We introduce the tools to design an approximation path.
%%%%%%%%%%%%%%%%%%%%%%%%%%%%%%%%%%%%%%%%%%%%%%%%%%%%%%%%%%%%%%%%%%%%%%%%%%%%%%%
\subsection{Preliminary results and technical tools}
\label{sub:preliminary_results_and_technical_tools}
%%%%%%%%%%%%%%%%%%%%%%%%%%%%%%%%%%%%%%%%%%%%%%%%%%%%%%%%%%%%%%%%%%%%%%%%%%%%%%%

\begin{definition}\label{def:general_convexity_smoothness}
Given a differentiable function $f$ and $x \in \dom f$, let $\mathcal{U}_{f,x}(\cdot)$ and $\mathcal{V}_{f,x}(\cdot)$ be non negative functions that vanish at 0. We say that $f$ is $\mathcal{U}_{f,x}$-convex (resp. $\mathcal{V}_{f,x}$-smooth) at $x$ when Inequality~\eqref{eq:mu_convexity} (resp. \eqref{eq:nu_smoothness}) is satisfied for any $z \in \dom f$
\begin{align}
\label{eq:mu_convexity} \mathcal{U}_{f,x}(z-x) &\leq f(z) - f(x) - \langle \nabla f(x), z - x\rangle \enspace,\\
\label{eq:nu_smoothness} \mathcal{V}_{f,x}(z-x) &\geq f(z) - f(x) - \langle \nabla f(x), z - x\rangle \enspace.
\end{align}
\end{definition}

This extends $\mu$-strong convexity and $\nu$-smoothness \citep{Nesterov04} and encompasses smooth uniformly convex losses
and generalized self-concordant ones.
\paragraph{Smooth uniformly convex case:} In this case, we have
\begin{align*}
\mathcal{U}_{f,x}(z-x) = \mathcal{U}(\normin{z-x}),   \\%\,
\mathcal{V}_{f,x}(z-x) = \mathcal{V}(\normin{z-x}),
\end{align*}
where $\mathcal{U}(\cdot)$ and $\mathcal{V}(\cdot)$ are increasing from $[0, +\infty)$ to $[0, +\infty]$ vanishing at $0$; see \citet{Aze_Penot_95}.
Examples of such functions are $\mathcal{U}(t) = \frac{\mu}{d} t^d$ and $\mathcal{V}(t) = \frac{\nu}{d} t^d$ where $d$, $\mu$ and $\nu$ are positive constants.
The case $d=2$ corresponds to strong convexity and smoothness; in general they are called \emph{uniformly convex of order $d$}, see \citep{Juditski_Nesterov14} or \citep[Ch. 10.2 and 18.5]{Bauschke_Combettes11} for details.

\paragraph{Generalized self-concordant case:}
a $\mathcal{C}^3$ convex function $f$ is $(M_{f}, \nu)$-generalized self-concordant of order $\nu \geq 2$ and $M_f\geq 0$ if  $\forall x \in \dom f$ and $\forall u, v \in \bbR^n$:
$$\left| \langle \nabla^3 f(x)[v]u,u \rangle\right| \leq M_f \norm{u}_{x}^{2}\norm{v}_{x}^{\nu-2}\norm{v}_{2}^{3-\nu}.$$

In this case, \citet[Proposition 10]{Sun_Tran-Dinh17} have shown that one could write:
\begin{align*}
\mathcal{U}_{f,x}(y-x)& = w_{\nu}(-d_{\nu}(x, y))\norm{y-x}_{x}^{2}\enspace,\\
\mathcal{V}_{f,x}(y-x) &= w_{\nu}(d_{\nu}(x, y))\norm{y-x}_{x}^{2}\enspace,
\end{align*}
where the last equality holds if $d_{\nu}(x,y)<1$ for the case $\nu>2$. Closed-form expressions of $w_{\nu}(\cdot)$ and $d_{\nu}(\cdot)$ are recalled in Appendix for logistic, quadratic and power losses.

\paragraph{Approximating the duality gap path.}
Assume we have constructed primal/dual feasible vectors for a finite grid of parameters $\Lambda_{T}=\{\lambda_{0},\dots, \lambda_{T-1}\}$, \ie we have at our disposal $(\beta^{(\lambda_{t})}, \theta^{(\lambda_{t})})$ for all $\lambda_t \in \Lambda_{T}$.
Let us denote $\Gap_{t} = \Gap_{\lambda_{t}}(\beta^{(\lambda_{t})}, \theta^{(\lambda_{t})})$, and for $\zeta_{t}=-\lambda_{t}\theta^{(\lambda_{t})}$,
$\Delta_{t} = f(X\beta^{(\lambda_{t})}) - f(\nabla f^*(\zeta_{t}))$.
For any function $\phi : \bbR^n \to [0, +\infty]$ that vanishes at $0$, $\rho \in \bbR$, we define
\begin{equation}
\label{eq:defined_phi}
Q_{t, \phi}(\rho) = \Gap_{t} + \rho \cdot (\Delta_{t} - \Gap_{t}) + \phi(-\rho \cdot \zeta_{t})\enspace.
\end{equation}
The terms $\Gap_t$ and $\Delta_t$ represent a measure of the optimization error at $\lambda_t$. The notation introduced in \eqref{eq:defined_phi} will be convenient to write concisely upper and lower bounds on the duality gap. This is the goal of the next lemma which leverages regularity of the loss function $f$, as introduced in \Cref{def:general_convexity_smoothness}.
This provides control on how the duality gap deviates when one evaluates it for another (close) parameter $\lambda$.

\begin{lemma}%[Bounding the warm start error]
\label{lm:tracking_gap_regularization_parameter}
We assume that $-\lambda \theta^{(\lambda_{t})} \in \dom f^*$ and $X^{\top} \theta^{(\lambda_{t})} \in \dom\Omega^{*}$. If $f^*$ is $\mathcal{V}_{f^*}$-smooth (resp. $\mathcal{U}_{f^*}$-convex)\footnote{we  drop $x$ in $\mathcal{U}_{f,x}$ and write $\mathcal{U}_{f}$ if no ambiguity holds.}, then for $\rho = 1-{\lambda}/{\lambda_{t}}$, the right (resp. left) hand side of Inequality~\eqref{eq:bounding_warm_start_gap} holds true
\begin{equation}
Q_{t, \mathcal{U}_{f^*}}(\rho)
\leq \Gap_{\lambda}(\beta^{(\lambda_{t})}, \theta^{(\lambda_{t})})
\leq Q_{t, \mathcal{V}_{f^*}}(\rho) \enspace. \label{eq:bounding_warm_start_gap}
\end{equation}
\end{lemma}
\begin{proof}
Proof for this result and for other propositions and theorems are deferred to the Appendix.
\end{proof}

The function $\phi$, chosen as $\mathcal{V}_{f^*}$ (resp. $\mathcal{U}_{f^*}$) for the upper (resp. lower) bound, essentially captures the regularity needed to approximate the duality gap at $\lambda$ when using primal/dual vector $(\beta^{(\lambda_{t})}, \theta^{(\lambda_{t})})$ for $\lambda_t$ close to $\lambda$.
When the function satisfies both inequalities, tightness of the bounds can be related to the conditioning $\mathcal{U}_{f^*} / \mathcal{V}_{f^*}$ of the dual loss $f^*$.
Equality holds for $\mathcal{U}_{f^*} \equiv \mathcal{V}_{f^*} \equiv \tfrac{1}{2}{\norm{\cdot}_{2}^{2}}$ (least-squares), showing the tightness of the bounds.

From \Cref{lm:tracking_gap_regularization_parameter}, we have $\Gap_{\lambda}(\beta^{(\lambda_{t})}, \theta^{(\lambda_{t})}) \leq \epsilon$ as soon as $Q_{t, \mathcal{V}_{f^*}}(\rho) \leq \epsilon$ where $\rho = 1 - \lambda/\lambda_{t}$ varies with $\lambda$.
Hence, we obtain the following proposition that allows to track the regularization path for an arbitrary precision on the duality gap. It proceeds by choosing the largest $\rho=\rho_t$ such that the upper bound in \Cref{eq:bounding_warm_start_gap} remains below $\epsilon$ and leads to \Cref{alg:adaptive_training_path} for computing an $\epsilon$-path.

%XXX TODO : here the impact of \epsilon_c is not clear. Also the proof is missing.
\begin{proposition}[Grid for a prescribed precision]\label{prop:grid_for_a_prescribed_precision} $\,$\\
Given $(\beta^{(\lambda_{t})}, \theta^{(\lambda_{t})})$ such that $\Gap_t \leq \epsilon_c < \epsilon$, for all
$\lambda \in \lambda_{t} \times \left[1 - \rho_{t}^{\ell}(\epsilon),\, 1 + \rho_{t}^{r}(\epsilon)\right]$,
we have $\Gap_{\lambda}(\beta^{(\lambda_{t})}, \theta^{(\lambda_{t})}) \leq \epsilon$ where $\rho_{t}^{\ell}(\epsilon)$ (resp. $\rho_{t}^{r}(\epsilon)$) is the largest non-negative $\rho$ s.t. $Q_{t, \mathcal{V}_{f^*}}(\rho) \leq \epsilon$ (resp. $Q_{t, \mathcal{V}_{f^*}}(-\rho) \leq \epsilon$).
\end{proposition}
Conversely, given a grid\footnote{we assume a decreasing order $\lambda_{t + 1}<\lambda_{t}$, reflecting common practices for GLM, \eg for the Lasso.} of $T$ parameters $\Lambda_{T}=\{\lambda_{0},\dots, \lambda_{T-1}\}$, we define $\epsilon_{\Lambda_{T}}$, the error of the approximation path on $[\lambdamin, \lambdamax]$ by using a piecewise constant approximation of the map $\lambda \mapsto \Gap_{\lambda}(\beta^{(\lambda_t)}, \theta^{(\lambda_t)})$:
\begin{equation}\label{eq:def_epsilon_lambda_T}
\epsilon_{\Lambda_{T}} := \max_{\lambda \in [\lambdamin, \lambdamax]} \min_{\lambda_{t} \in \Lambda_{T}}\Gap_{\lambda}(\beta^{(\lambda_{t})}, \theta^{(\lambda_{t})}) \enspace.
\end{equation}
This error is however difficult to evaluate in practice so we rely on a tight upper bound based on \Cref{lm:tracking_gap_regularization_parameter} that often leads to closed-form expressions.

\begin{proposition}[Precision for a given grid]\label{prop:precision_for_a_given_grid}
Given a grid of parameters $\Lambda_{T}$, the set $\{\beta^{(\lambda)} : \lambda \in \Lambda_T\}$ is an $\epsilon_{\Lambda_{T}}$-path with $\epsilon_{\Lambda_{T}} \leq \max_{t \in [T]} Q_{t,\mathcal{V}_{f^*}}(1-\lambda_{t}^{\star}/\lambda_{t}) $ where for all $t \in \{0,\dots, T-1\}$, $\lambda_{t}^{\star}$ is the largest $\lambda \in [\lambda_{t + 1},\lambda_{t}]$ such that $Q_{t,\mathcal{V}_{f^*}}(1-\lambda/\lambda_{t}) \geq Q_{t+1,\mathcal{V}_{f^*}}(1-\lambda/\lambda_{t + 1})$.
\end{proposition}
% XXX TODO: for revision add a graph explaining this part. too obscure at first readding.
\paragraph{Construction of dual feasible vector.}
We rely on gradient rescaling to produce a dual feasible vector:

\begin{lemma}\label{lem:dual_point}
For any $\beta^{(\lambda_{t})} \in \bbR^p$, the vector
 \begin{equation*}
 \label{eq:def_dual_vector}
 \theta^{(\lambda_{t})} =\frac{-\nabla f (X\beta^{(\lambda_{t})})}{\max(\lambda_{t}, \sigma_{\dom \Omega^*}^{\circ} (X^\top \nabla f(X\beta^{(\lambda_{t})}))}\enspace,
\end{equation*}
is feasible: $-\lambda \theta^{(\lambda_{t})} \in \dom f^*$, $X^{\top} \theta^{(\lambda_{t})} \in \dom \Omega^{*}$.
\end{lemma}
\begin{remark}
When the regularization is a norm, $\Omega(\cdot) = \norm{\cdot}$ then $\sigma_{\dom \Omega^*}^{\circ}$ is the associated dual norm $\norm{\cdot}_*$.
\end{remark}

The dual $\theta^{(\lambda_{t})}$ in \Cref{lem:dual_point} implies that $\Gap_t$ and $\Delta_{t}$ converge to $0$ when $\beta^{(\lambda_{t})}$ converges to $\hat \beta^{(\lambda_{t})}$ \citep{Ndiaye_Fercoq_Gramfort_Salmon17}. 
% As an optimization accuracy for the approximation path, we can consider the criterion $\max(\Gap_t, \Delta_{t}) \leq \epsilon_c$.
% of Problem~\eqref{eq:primal_problem}.

% \eugene{peut etre rajouter la preuve dans l'appendix? ou faire une reference anonyme? car c'est fait dans ma these}
% \jrs{yes, to be done for revision...je laisse des XXX dans le tex pour cela.}
% XXX TODO for revision: the proof is needed in appendix. TO ADD!

\paragraph{Finding $\rho$.}
Following \Cref{prop:grid_for_a_prescribed_precision}, a 1-dimensional equation $Q_{t,\mathcal{V}_{f^*}}(\rho) = \epsilon$ needs to be solved to obtain an $\epsilon$-path.
This can be done efficiently at high precision by numerical solvers if no explicit solution is available.

As a corollary from \Cref{lm:tracking_gap_regularization_parameter} and \Cref{prop:precision_for_a_given_grid}, we recover the analysis by \citet{Giesen_Laue_Mueller_Swiercy12}:

\begin{corollary}\label{cor:quadratic_step_size}
If the function $f^*$ is $\frac{\nu}{2}\normin{\cdot}^2$-smooth, the left ($\rho^{\ell}_{t}$) and right ($\rho^{r}_{t}$) step sizes defined in \Cref{prop:grid_for_a_prescribed_precision} have closed-form expressions:
\begin{align*}
\rho^{\ell}_{t} =
\frac{\! \sqrt{\! 2 \nu \delta_{t} \norm{\zeta_{t}}^{2} \!+ {\tilde{\delta}_{t}}^{2}} - \tilde{\delta}_{t}}{\nu \norm{\zeta_{t}}^{2}},
\rho^{r}_{t} & =
 \frac{\! \sqrt{ \! 2 \nu\delta_{t} \norm{\zeta_{t}}^{2} \!+ {\tilde{\delta}_{t}}^{2}} + \tilde{\delta}_{t}}{\nu \norm{\zeta_{t}}^{2}},
\end{align*}
where $\delta_{t} := \epsilon - \Gap_{t}$ and $\tilde{\delta}_{t} := \Delta_{t} - \Gap_{t}$. This is simplified to $\delta_{t} = \epsilon - \epsilon_c$ and $\tilde{\delta}_{t}=0$ when $\max(\Gap_t, \Delta_{t}) \leq \epsilon_c$.
 % and $\nu=\nu_{f^*}$ is the smoothness constant of the dual loss $f^*$.
\end{corollary}

% \jrs{naming must be change: remove GLMNET in Fig3 a and b; also legend is too vague...especially as figure is far from description.}
% \jrs{also T could appear somwhere in the legend of the figure to reconnec with the text for the size of the grid. And to be coherent, if we report the time w.r.t. to the default grid in the top, we should do the same with the grid size as a percentage too. }
% \jrs{and adding the default grid here again might help too: }

\begin{algorithm}[!t]
\caption{
\texttt{training\_path}}
\label{alg:adaptive_training_path}
\begin{algorithmic}
{\small\STATE {\bfseries Input:} $f, \Omega, \epsilon, \epsilon_{c}, [\lambda_{\min}, \lambda_{\max}]$
\STATE Initialization: $t=0$, $\lambda_0 = \lambda_{\max}$, $\Lambda = \{\lambda_{\max} \}$
%   \STATE If uniform
\REPEAT
%   \STATE If unilateral
\STATE Get $\beta^{(\lambda_{t})}$ solving~\eqref{eq:primal_problem}
 % for $\lambda = \lambda_{t}$ 
 to accuracy $\Gap_t \leq \epsilon_c < \epsilon$
% \STATE Compute $\displaystyle \rho_{t}^{\ell}(\epsilon)= \max \{ \rho \text{ s.t. } Q_{t,\mathcal{V}_{f^*}}(\rho) \leq \epsilon \}$
\STATE Compute the step size $\rho_{t}^{\ell}(\epsilon)$ following Proposition \ref{prop:unilateral}, \ref{prop:bilateral}, \ref{prop:uniform}.

%\STATE IF bilateral
\STATE Set $\lambda_{t + 1} = \max(\lambda_{t} \times (1 - \rho_{t}^{\ell}),\, \lambda_{\min})$
\STATE $\Lambda \leftarrow \Lambda \cup \{\lambda_{t + 1} \}$ and $t \leftarrow t+1$
\UNTIL{$\lambda_{t} \leq \lambdamin$}
\STATE {\bfseries Return:} $\{\beta^{(\lambda_{t})}\;:\; \lambda_{t} \in \Lambda \}$}
\end{algorithmic}
\end{algorithm}

%%%%%%%%%%%%%%%%%%%%%%%%%%%%%%%%%%%%%%%%%%%%%%%%%%%%%%%%%%%%%%%%%%%%%%%%%%%%%%%
\subsection{Discretization strategies}
\label{subsec:Discretization_strategies}
%%%%%%%%%%%%%%%%%%%%%%%%%%%%%%%%%%%%%%%%%%%%%%%%%%%%%%%%%%%%%%%%%%%%%%%%%%%%%%%

We now establish new strategies for the exploration of the hyperparameter space in the search for an $\epsilon$-path.

For regularized learning methods, it is customary to start from a large regularizer\footnote{for the Lasso one often chooses $\lambda_0=\lambdamax:=\norm{X^\top y}_{\infty}$} $\lambda_{0} = \lambdamax$ and then to perform the computation of $\tbeta{\lambda_{t+1}}$ after the one of $\tbeta{\lambda_{t}}$, until the smallest parameter of interest $\lambdamin$ is reached.
Models are generally computed by increasing complexity, allowing important speed-ups due to \emph{warm start} \citep{Friedman_Hastie_Hofling_Tibshirani07} when the $\lambda$'s are close to each other.
Knowing $\lambda_t$, we provide a recursive strategy to construct $\lambda_{t+1}$.

\paragraph{Adaptive unilateral.}

The strategy we call \emph{unilateral} consists in computing the new parameter as $\lambda_{t+1}=\lambda_{t} \times (1 - \rho_{t}^{\ell}(\epsilon))$ as in \Cref{prop:grid_for_a_prescribed_precision}.

\begin{proposition}[Unilateral approximation path]\label{prop:unilateral}
Assume that $f^*$ is $\mathcal{V}_{f^*}$-smooth.
We construct the grid of parameters $\Lambda^{(u)}(\epsilon) = \{\lambda_0, \ldots, \lambda_{T_{\epsilon}-1} \}$ by
$$\lambda_0 = \lambda_{\max},\quad \lambda_{t + 1} = \lambda_{t} \times (1 - \rho_{t}^{\ell}(\epsilon)) \enspace,$$
 and $(\beta^{(\lambda_{t})}, \theta^{(\lambda_{t})})$ s.t. $\Gap_t \leq \epsilon_{c} < \epsilon$ for all $t$.
Then, the set $\{\beta^{(\lambda_{t})} \;:\; \lambda_{t} \in \Lambda^{(u)}(\epsilon)\}$ is an $\epsilon$-path for Problem~\eqref{eq:primal_problem}.
\end{proposition}
% \jo{the naming should be consistent in \Cref{fig:approximation_gap_lasso}. to double check...}
This strategy is illustrated in \Cref{subfig:approximation_gap_lasso_unit} on a Lasso case, and stands as a generic one to compute an approximation path for loss functions satisfying assumptions in \Cref{def:general_convexity_smoothness}.

% \begin{figure}[t]
% \center
% \includegraphics[width=0.8\columnwidth, keepaspectratio, trim=1em 2em 1em 1.5em, clip]{approximation_path}
% \caption{Illustration of the construction of an $\epsilon$-path for the Lasso  on synthetic dataset generated with $\texttt{sklearn}$ as $X, y =$ \texttt{make\_regression}$(n=30, p=150)$ at accuracy $\epsilon=\norm{y}_{2}^{2}/40$ and $\epsilon_{c}= \epsilon/10$. We choose $\lambda_{\max} = \normin{X^\top y}_{\infty}$ and $\lambda_{\min}=\lambda_{\max}/10$ leading to a path complexity $T_{\epsilon}=6$. For Lasso the bound is piece-wise quadratic.\label{fig:approximation_gap_lasso} \jrs{unilateral case}}
% \end{figure}

\begin{figure*}[t]
\centering
\subfigure[$\ell_1$ least-squares regression on climate data set \texttt{NCEP/NCAR Reanalysis}; $n=814$ observations and $p=73577$ features.]{\includegraphics[width=0.48\linewidth, keepaspectratio]{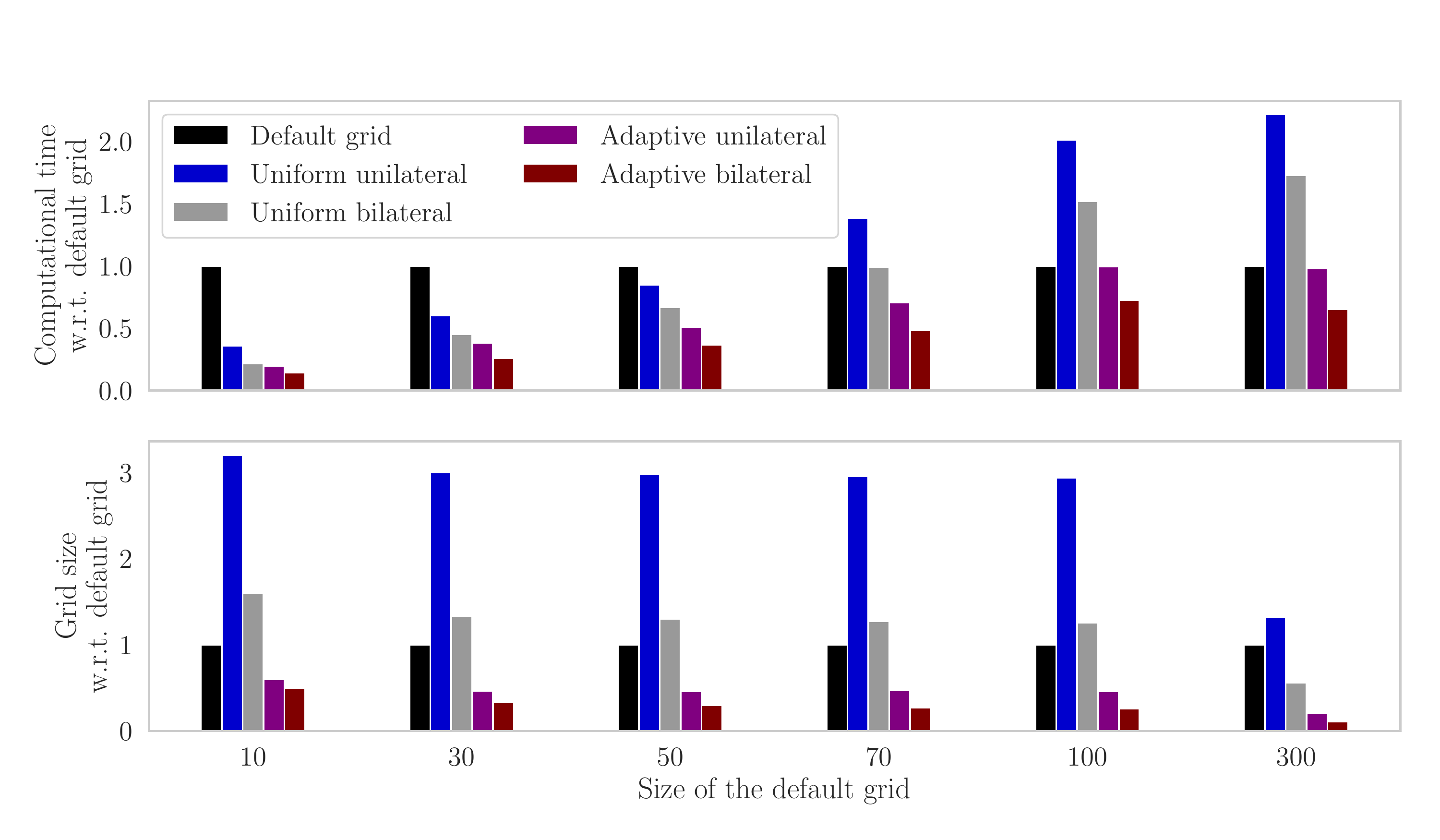}}\label{subfig:bench_Lasso_climate_grid_n_lambdas_tau10}
\hspace{0.25cm}
\subfigure[$\ell_1$ logistic regression on \texttt{leukemia} data set with $n=72$ observations and $p=7129$ features.]{\includegraphics[width=0.48\linewidth, keepaspectratio]{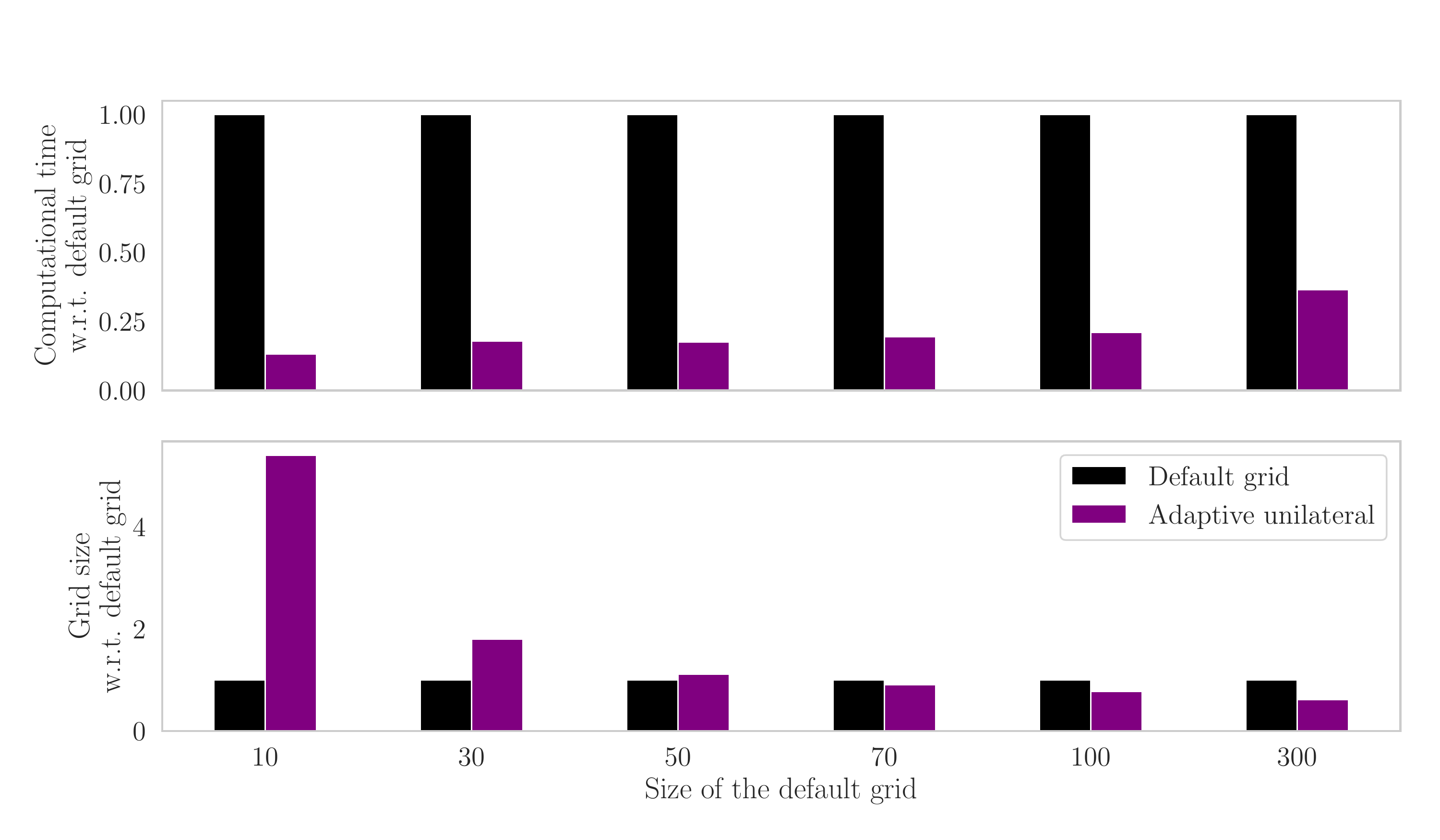}}\label{subfig:bench_Logreg_leukemia_grid_n_lambdas_tau10}
\caption{Computation of the approximation path to reach the same error than the default grid
($\epsilon = 10^{-4} \norm{y}^2$ for the least-squares case and $\epsilon = 10^{-4} \min(n_1, n_2)/n$ where $n_i$ is the number of observations in the class $i \in \{0, 1\}$, for the logistic case). We have used the same (vanilla) coordinate descent optimization solver with warm start between parameters for all grids. Note that a smaller grid do not imply faster computation, as the interplay with the warm-start can be intricate in our sequential approach.} \label{fig:bench_approximation_path}
\end{figure*}

\paragraph{Adaptive bilateral.}
For uniformly convex functions, we can make a larger step by combining the information given by the left and right step sizes.
Indeed, let us assume that we explore the parameter range  $[\lambdamin, \lambdamax]$.
Starting from a parameter $\lambda_{t}$, we define the next step, given by \Cref{prop:grid_for_a_prescribed_precision}, $\lambda_{t}^{\ell} := \lambda_{t}(1 - \rho_{t}^{\ell})$. Then it exists $\lambda_{t'} \leq \lambda_{t}^{\ell}$ such that $\lambda_{t'}^{r} := \lambda_{t'}(1 + \rho_{t'}^{r}) = \lambda_{t}^{\ell}$. Thus a larger step can be done by using
$\lambda_{t'} = \lambda_{t} \times {(1 - \rho_{t}^{\ell})}/{(1 + \rho_{t'}^{r})}$. However $\rho_{t'}^{r}$ depends on the (approximated) solution $\beta^{(\lambda_{t'})}$ that we do not know before optimizing the problem for parameter $\lambda_{t'}$ when computing sequentially the grid points in decreasing order \ie $\lambda_{t'} \leq \lambda_{t}$.
We overcome this issue in \Cref{lm:bounds_gradient_and_gap} by (upper) bounding all the constants in $Q_{t',\mathcal{V}_{f^*}}(\rho)$ that depend on the solution $\beta^{(\lambda_{t'})}$, by constants involving only information given by $\beta^{(\lambda_{t})}$.

\begin{lemma}\label{lm:bounds_gradient_and_gap}
Assuming $f$ uniformly smooth yields $\normin{\nabla f(X\beta^{(\lambda_{t'})})}_* \leq \widetilde R_{t}$, where $\widetilde R_{t} := {\mathcal{V}_{f}^{*}}^{-1}\big(f(X\beta^{(\lambda_{t})}) + \frac{2\epsilon_{c}}{\rho_{t}^{\ell} (\epsilon)}\big)$.
If additionally $f$ is uniformly convex, this yields $\Delta_{t'} \leq \widetilde \Delta_{t}$, where $\widetilde \Delta_{t} := \widetilde R_{t} \times \mathcal{U}_{f}^{-1}(\epsilon_{c})$ as well as $\Gap_{\lambda}(\beta^{(\lambda_{t'})}, \theta^{(\lambda_{t'})}) \leq Q_{t',\mathcal V_{f^*}}(\rho) \leq \widetilde Q_{t, \mathcal V_{f^*}}(\rho)$, where
\begin{equation*}
% \label{bound_on_gradients}
\widetilde Q_{t, \mathcal V_{f^*}}(\rho) =  \epsilon_{c} + \rho \cdot (\widetilde \Delta_{t} - \epsilon_{c} ) + \mathcal{V}_{f^*}\left(\left|\rho\right|\cdot\widetilde R_{t}\right) \enspace.
\end{equation*}
\end{lemma}

Let us now define $\rho_{t}^{(b)}(\epsilon) = \displaystyle\frac{\rho_{t}^{\ell}(\epsilon) + \tilde \rho_{t}^{r}(\epsilon)}{1 + \tilde \rho_{t}^{r}(\epsilon)} $, where $\rho_{t}^{\ell}(\epsilon)$ is defined in \Cref{prop:grid_for_a_prescribed_precision} and $\tilde \rho_{t}^{r}(\epsilon)$ is the largest non negative $\rho$ such that $\widetilde Q_{t, \mathcal V_{f^*}}(\rho) \leq \epsilon$ in \Cref{lm:bounds_gradient_and_gap}.

\begin{proposition}[Bilateral Approximation Path]\label{prop:bilateral}
Assume that $f$ is uniformly convex and smooth.
We construct the grid $\Lambda^{(b)}(\epsilon) = \{\lambda_0, \ldots, \lambda_{T_{\epsilon}-1} \}$ by
$$\lambda_0 = \lambda_{\max},\quad  \lambda_{t + 1} = \lambda_{t} \times (1 - \rho_{t}^{(b)}(\epsilon)) \enspace,$$
 and $(\beta^{(\lambda_{t})}, \theta^{(\lambda_{t})})$ s.t. $\Gap_t \leq \epsilon_{c} < \epsilon$ for all $t$.
Then the set $\{\beta^{(\lambda_{t})} \;:\; \lambda_{t} \in \Lambda^{(b)}(\epsilon)\}$ is an $\epsilon$-path for Problem~\eqref{eq:primal_problem}.
\end{proposition}
% \jo{to add here a word on the fact that now we illustrate also the bilateral (adaptive case) case.}
This strategy is illustrated in \Cref{subfig:approximation_gap_lasso_bilat} on a Lasso example.

\paragraph{Uniform unilateral and bilateral.}
In some cases, it may be advantageous to have access to a predefined grid before launching a hyperparameter selection procedure such as \texttt{hyperband} \citep{Li_Jamieson_DeSalvo_Rostamizadeh_Talwakar17} or for parallel computations. Given the initial information from the initialization $(\beta^{(\lambda_0)}, \theta^{(\lambda_0)})$, we can build a uniform grid that guarantees an $\epsilon$-approximation before solving any optimization problem.
Indeed, by applying \Cref{lm:bounds_gradient_and_gap} at $t=0$, we have $\Gap_{\lambda}(\beta^{(\lambda_{t})}, \theta^{(\lambda_{t})}) \leq \widetilde Q_{0, \mathcal{V}_{f^*}}(\rho)$. We can define $\widetilde \rho_{0}^{\ell}(\epsilon)$ (resp. $\widetilde \rho_{0}^{r}(\epsilon)$) as the largest non-negative $\rho$ s.t. $\widetilde Q_{0, \mathcal{V}_{f^*}}(\rho)\leq \epsilon$ (resp. $\widetilde Q_{0, \mathcal{V}_{f^*}}(-\rho)\leq \epsilon$) and also
\begin{equation}
\rho_{0}(\epsilon) =
\begin{cases}
\widetilde\rho_{0}^{\ell}(\epsilon) &\text{for \emph{unilateral} path},\\
\frac{\widetilde\rho_{0}^{\ell}(\epsilon) + \widetilde\rho_{0}^{r}(\epsilon)}{1 + \tilde \rho_{0}^{r}(\epsilon)} &\text{for \emph{bilateral} path}.
\end{cases}
\end{equation}

\begin{proposition}[Uniform approximation path]$\,$\\ \label{prop:uniform}
Assume that $f$ is uniformly convex and smooth, and define the grid $\Lambda^{(0)}(\epsilon) = \{\lambda_0, \ldots, \lambda_{T_{\epsilon}-1} \}$ by
$$\lambda_0 = \lambda_{\max},\quad  \lambda_{t + 1} = \lambda_{t} \times (1 - \rho_{0}(\epsilon))\enspace,$$
and $\forall t \in [T], (\beta^{(\lambda_{t})}, \theta^{(\lambda_{t})})$ s.t. $\Gap_t \leq \epsilon_{c} < \epsilon$.
Then the set $\{\beta^{(\lambda_{t})} \;:\; \lambda_{t} \in \Lambda^{(0)}(\epsilon)\}$ is an $\epsilon$-path for Problem~\eqref{eq:primal_problem}
 with at most $T_{\epsilon}$ grid points where
 \begin{equation}
 T_{\epsilon} = \left\lfloor \frac{\log({\lambdamin}/{\lambdamax})}{\log(1 - \rho_{0}(\epsilon))} \right\rfloor \enspace.
 \end{equation}
% \jrs{$$T_{\epsilon} = \left\lfloor \frac{\log({\lambdamax}/{\lambdamin})}{\log\big(1/(1 - \rho_{0}(\epsilon))\big)} \right\rfloor \enspace.$$}
\end{proposition}
% \begin{remark}
% For cases such as the Lasso, an explicit value of $\beta^{(\lambdamax)}$ can be obtained, say $\beta^{(\lambda_0)}=0$, for $\lambda_0=\norm{X^\top y}_{\infty}$.
% Since the uniform grid depends only on this initially known value $\beta^{(\lambda_0)}$, it can be computed prior any optimization solver is launched. Hence, it is well suited for parallel computation over the grid of parameters.
% \end{remark}

%%%%%%%%%%%%%%%%%%%%%%%%%%%%%%%%%%%%%%%%%%%%%%%%%%%%%%%%%%%%%%%%%%%%%%%%%%%%%%%
\subsection{Limitations of previous framework}
%%%%%%%%%%%%%%%%%%%%%%%%%%%%%%%%%%%%%%%%%%%%%%%%%%%%%%%%%%%%%%%%%%%%%%%%%%%%%%%

Previous algorithms for computing $\epsilon$-paths have been initially developed with a complexity of $O(1/\epsilon)$ \citep{Clarkson10, Giesen_Jaggi_Laue10} in a large class of problems.
Yet, losses arising in machine learning have often nicer regularities that can be exploited.
This is all the more striking in the Lasso case where a better complexity in $O(1/\sqrt{\epsilon})$ was obtained by \citet{Mairal_Yu12, Giesen_Laue_Mueller_Swiercy12}.

The relation between path complexity and regularity of the objective function remains unclear and previous methods do not apply to all popular learning problems.
For instance  the dual loss $f^*$ of the logistic regression is not uniformly smooth.
So to apply the previous theory, one needs to optimize on a (potentially badly pre-selected) compact set.

Let us consider the one dimensional toy example where $p=1$, $\beta\in\bbR$, $X=\mathrm{Id}_p$, $y=-1$ and the loss function $f(X\beta) = \log(1+\exp(\beta))$.
We have, $\nabla^2 f(\beta) = \exp(\beta)/(1 + \exp(\beta))^2$.
Then for Problem~\eqref{eq:primal_problem}, since $P_{\lambda}(\tbeta{\lambda}) \leq P_{\lambda}(0)$, we have $|\tbeta{\lambda}| \in [0, \log(2)/\lambda]$ and a smoothness constant $\nu_{f^*} \approx \exp(\lambda)$ for the dual can be estimated at each step.
This leads to an unreasonable algorithm with tiny step sizes in \Cref{cor:quadratic_step_size}.
Also, the algorithm proposed by \citet{Giesen_Laue_Mueller_Swiercy12} can not be applied for the logistic loss since the dual function is not polynomial.

Our proposed algorithm does not suffer from such limitations and we introduce a finer analysis that takes into account the regularity of the loss functions.

%%%%%%%%%%%%%%%%%%%%%%%%%%%%%%%%%%%%%%%%%%%%%%%%%%%%%%%%%%%%%%%%%%%%%%%%%%%%%%%
\subsection{Complexity and regularity}
\label{subsec:Complexity_and_Regularity}
%%%%%%%%%%%%%%%%%%%%%%%%%%%%%%%%%%%%%%%%%%%%%%%%%%%%%%%%%%%%%%%%%%%%%%%%%%%%%%%

\paragraph{Lower bound on path complexity.}
For our method, the lower bound on the duality gap quantifies how close the  from \Cref{prop:grid_for_a_prescribed_precision} is from the best possible one can achieve for smooth loss functions.
Indeed, at optimal solution, we have $\Gap_{t} = \Delta_{t} = 0$.
Thus the largest possible step --- starting at $\lambda_t$ and moving in decreasing order --- is given by the smallest $\lambda \in [\lambda_{\min},\lambda_t]$ such that $\mathcal{U}_{f^*}(-\hat \zeta_t \times \rho) > \epsilon$ where $\hat \zeta_t = -\lambda_t \ttheta{\lambda_t}$.
Hence, \textit{any} algorithm for computing an $\epsilon$-path
% with the duality gap 
for $\mathcal{U}_{f^*}$-uniformly convex dual loss, have necessarily a complexity of order at least $O(1/\mathcal{U}_{f^*}^{-1}(\epsilon))$.

% Our framework has the noticeable advantage to naturally adapt to the regularity of the loss function and do not require specific algebra for each function as it was done previously in the literature.

\paragraph{Upper bounds.}
We remind that we write $T_{\epsilon}$ for the complexity of our proposed approximation path \ie the cardinality of the grid returned by \Cref{alg:adaptive_training_path}. In the following proposition, we propose a bound on the complexity \wrt the regularity of the loss function. Discussions on the constants and assumptions are provided in the Appendix.

\begin{proposition}[Approximation path: complexity] \quad \\
\label{prop:approx_path_complexity}
Assuming that $\max(\Gap_t, \Delta_t) \leq \epsilon_{c} < \epsilon$ at each step $t$, there exists an explicit constant $C_{f}(\epsilon_{c})>0$ such that
\begin{align}
T_{\epsilon} & \leq \log\Big(\frac{\lambdamax}{\lambdamin}\Big) \times \frac{C_{f}(\epsilon_{c})}{\mathcal{W}_{f^*}(\epsilon - \epsilon_{c})} \enspace,
\end{align}
where for all $t>0$, the function $\mathcal{W}_{f^*}$ is defined by
\begin{align*}
\mathcal{W}_{f^*} (\cdot)=\begin{cases}
\mathcal{V}_{f^*}^{-1}(\cdot), &  \!\!\!\text{if } f  \text{ is uniformly convex and smooth}\\
\sqrt{\cdot}, & \!\!\! \text{if } f \text{ is Generalized Self-Concordant} \\& \text{and uniformly-smooth.}
\end{cases}
\end{align*}
Moreover, $C_{f}(\epsilon_{c})$ is an uniform upper bound of $\norm{\zeta_t}_*$ along the path, that tends to a constant $C_f$ when $\epsilon_{c}$ goes to $0$.
\end{proposition}

\Cref{prop:approx_path_complexity} applied to the special case when $f$ is $\nu$-smooth and $\mu$-strongly convex reads $\log\Big(\frac{\lambdamax}{\lambdamin}\Big) \sqrt{\frac{\nu}{\mu} \frac{f(X\beta^{(\lambda_0)})}{\epsilon - \epsilon_{c}}}$ for the complexity for any data $X, y$. This is not explicitly dependent on the dimension $n$ and $p$ and are more scalable.

%%%%%%%%%%%%%%%%%%%%%%%%%%%%%%%%%%%%%%%%%%%%%%%%%%%%%%%%%%%%%%%%%%%%%%%%%%%%%%%
%%%%%%%%%%%%%%%%%%%%%%%%%%%%%%%%%%%%%%%%%%%%%%%%%%%%%%%%%%%%%%%%%%%%%%%%%%%%%%%
\section{Validation path}
\label{sec:validation_path}
%%%%%%%%%%%%%%%%%%%%%%%%%%%%%%%%%%%%%%%%%%%%%%%%%%%%%%%%%%%%%%%%%%%%%%%%%%%%%%%
%%%%%%%%%%%%%%%%%%%%%%%%%%%%%%%%%%%%%%%%%%%%%%%%%%%%%%%%%%%%%%%%%%%%%%%%%%%%%%%

\begin{algorithm}[!t]
   \caption{$\epsilon_v$-path for Validation Set}
   \label{alg:adaptive_validation_path}
\begin{algorithmic}
   \STATE {\bfseries Input:} $f, \Omega, \epsilon_{v}, [\lambda_{\min}, \lambda_{\max}]$ \\
   \STATE Compute $\epsilon_{v, \mu}$ as in \Cref{prop:Grid_for_a_prescribed_validation_error}
   \STATE $\Lambda(\epsilon_{v,\mu}) = \texttt{training\_path}\left(f, \Omega, \epsilon_{v,\mu}, [\lambda_{\min}, \lambda_{\max}]\right)$
   \STATE {\bfseries Return:} $\Lambda(\epsilon_{v,\mu})$
\end{algorithmic}
\end{algorithm}

To achieve good generalization performance, estimators defined as solutions of Problem~\ref{eq:primal_problem} require a careful adjustment of $\lambda$ to balance data-fitting and regularization.
A standard approach to calibrate such a parameter is to select it by comparing the validation errors on a finite grid (say with K-fold cross-validation).
Unfortunately, %the guarantees of this grid search method remain unclear \jrs{cite a reference or remove this judgment sentence} and
it is often difficult to determine a priori the grid limits, the number of $\lambda$'s (number of points in the grid) or how they should be distributed to achieve low validation error.
% , and practitioners often rely on rules of thumb.

Considering the validation data $({X'}, {y'})$ with $n'$ observations and loss\footnote{the data-fitting terms might differ from training to testing; for instance for logistic regression the $\ell_{0/1}$-loss is used for validation but the logistic function is optimized at training.} $\mathcal{L}$, we define the validation error for $\beta \in \bbR^p$:
\begin{equation}\label{eq:validation_error}
E_v(\beta) = \mathcal{L}({y'}, {X'} \beta) \enspace.
\end{equation}
For selecting a hyperparameter, we leverage our approximation path to solve the bi-level problem
\begin{align*}
\argmin_{\lambda \in [\lambda_{\min}, \lambda_{\max}]} &E_v(\tbeta{\lambda}) = \mathcal{L}(y', X'\tbeta{\lambda}) \\
&\text{s.t. } \tbeta{\lambda} \in \argmin_{\beta \in \bbR^p} f(X\beta) + \lambda \Omega(\beta) \enspace.
\end{align*}

Recent works have addressed this problem by using gradient-based algorithms, see for instance \citet{Pedregosa16, Franceschi_Frasconi_Salzo_Pontil18} who have shown promising results in computational time and scalability w.r.t. multiple hyperparameters.
Yet, they require assumptions such as smoothness of the validation function $E_v$ and non-singular Hessian of the inner optimization problem at optimal values which are difficult to check in practice since they depend on the optimal solutions $\tbeta{\lambda}$.
Moreover, they can only guarantee convergence to stationary point.

\begin{figure*}[!ht]
\centering
\subfigure[Synthetic data set generated using the \texttt{sklearn} command $X, y=$ \texttt{make\_sparse\_uncorrelated}$(n=30, p=50)$.]{\includegraphics[width=\columnwidth, keepaspectratio]{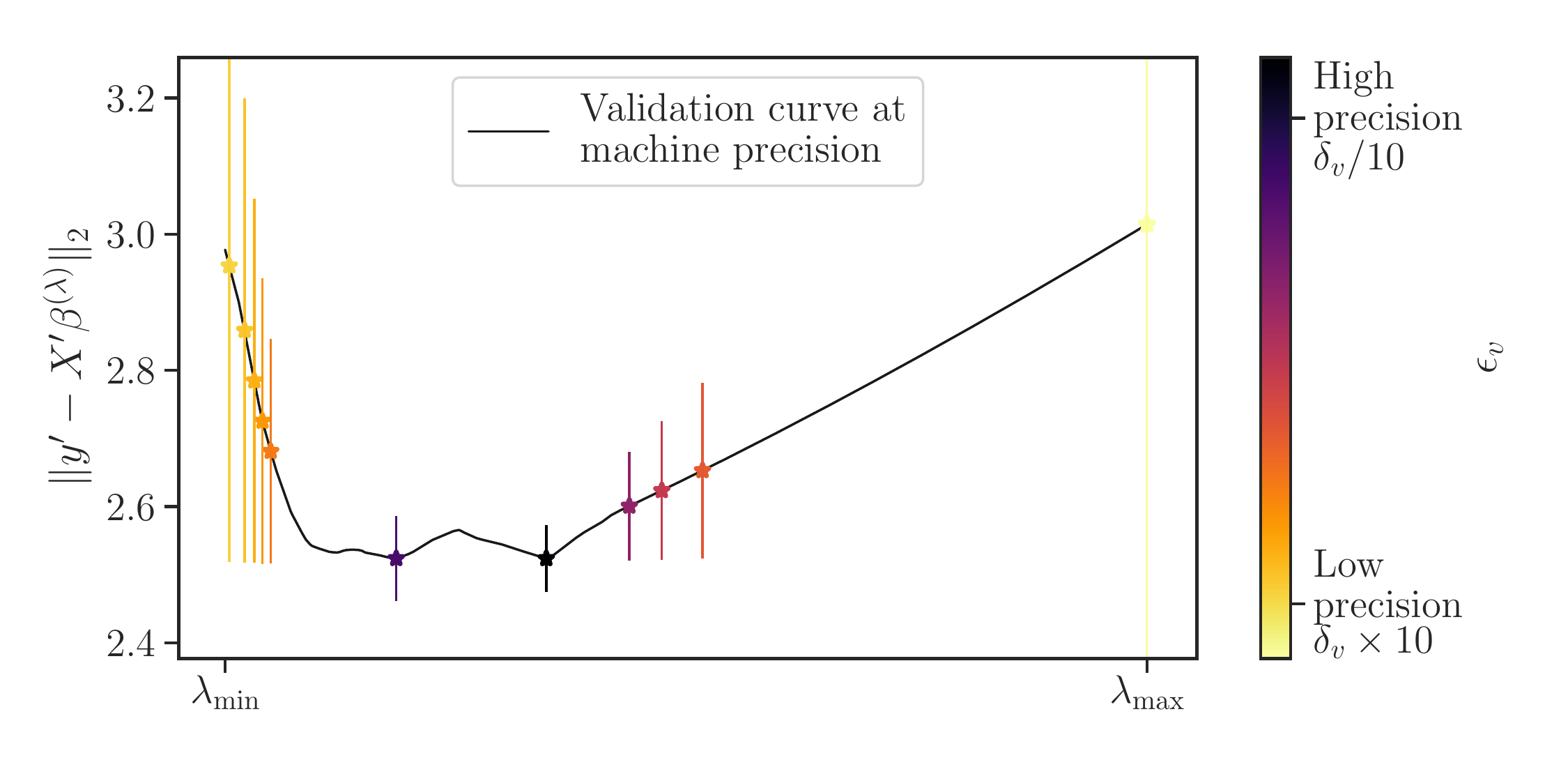}}\label{fig:validation_diabetes}
\hspace{0.25cm}
% [width=0.8\columnwidth, keepaspectratio, trim=1em 2em 1em 1.5em, clip]
\subfigure[Synthetic data set generated using the \texttt{sklearn} command $X, y =$ \texttt{make\_regression}$(n=500, p=5000)$.]{\includegraphics[width=\columnwidth, keepaspectratio]{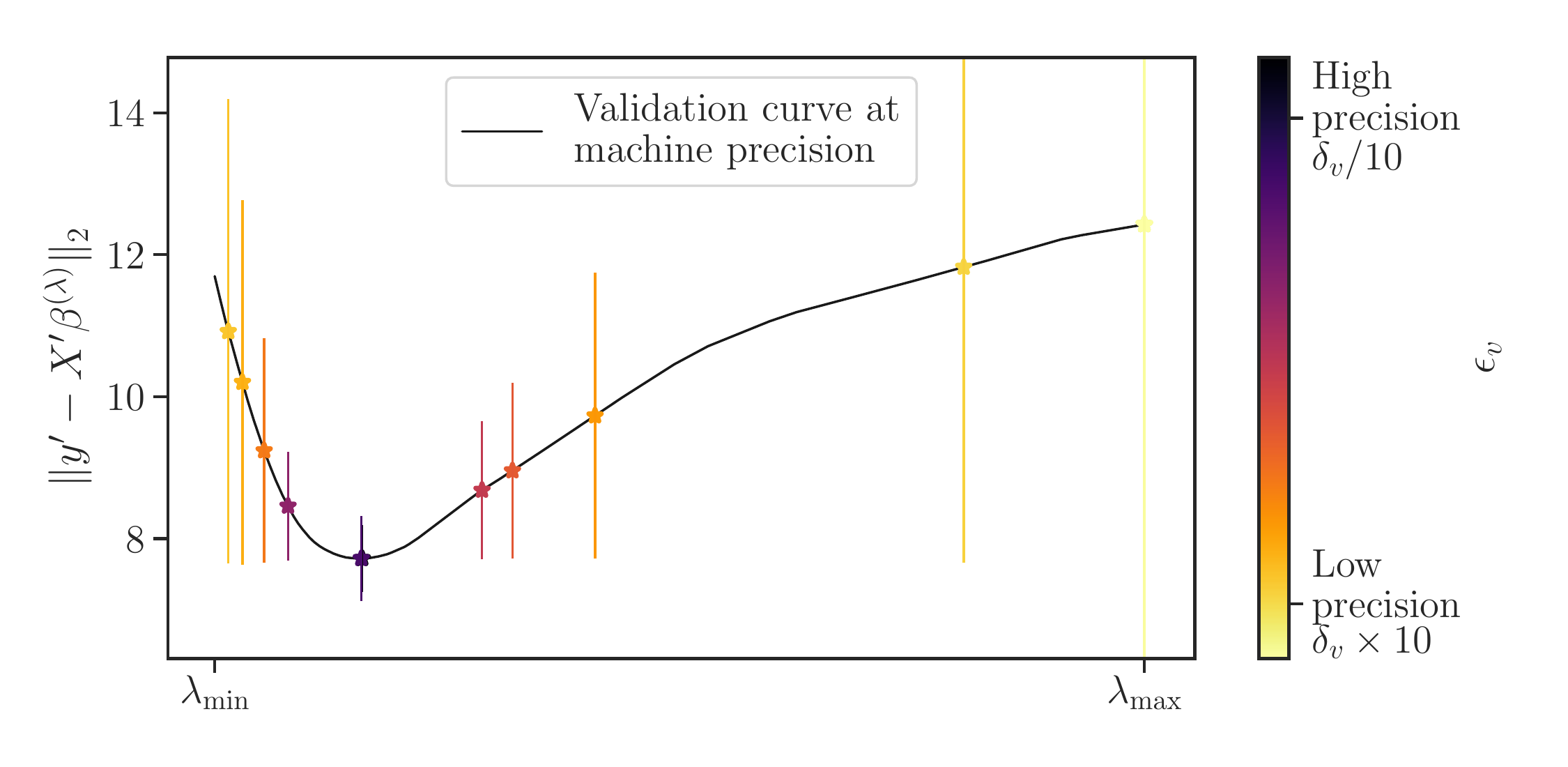}}\label{fig:validation_synthetic}
\caption{Safe selection of the optimal hyperparameter for Enet on the validation set (30\%  of the observations). The targeted accuracy $\epsilon_v$ is refined from $\delta_v \times 10$ to $\delta_v / 10$ with $\delta_v = \max_{\lambda_t \in \Lambda} E_v(\beta^{(\lambda_t)}) - \min_{\lambda_t \in \Lambda} E_v(\beta^{(\lambda_t)})$ and $\Lambda$ is the default grid between $\lambda_{\max}=\normin{X^\top y}_{\infty}$ and $\lambda_{\min} = \lambda_{\max}/100$ of size $T=200$. The stars represent the worst case solution amount the one generated by \Cref{alg:adaptive_validation_path} (with bilateral path). For loose precision suboptimal parameters are identified, but better ones are found as the accuracy $\epsilon_v$ decreases.}
\label{fig:bench_validation_path}
\end{figure*}

In this section, we generalize the approach of \cite{Shibagaki_Suzuki_Karasuyama_Takeuchi15} and show that with a safe and simple exploration of the parameter space, our algorithm has a global convergence property.
For that, we assume the following conditions on the validation loss and on the inner optimization objective throughout the section:
  {
% 
%       \theoremstyle{plain}
%       \newtheorem*{assumption1f}{A1}
%       % \newtheorem*{assumption2f}{A2}
% 
% \let\origtheassumption\theassumption
% 
% \edef\oldassumption{\the\numexpr\value{assumption1f}+1}
% \setcounter{assumption1f}{0}
\newtheorem{assumption}{A}
\let\origtheassumption\theassumption
\renewcommand\theassumption{A\arabic{assumption}}
  }
  %   {
  %     \theoremstyle{plain}
  %     \newtheorem{assumption}{A}
  % }
\leqnomode
\begin{align}\label{as:1}
\phantom{44}&|\mathcal{L}(a, b) - \mathcal{L}(a, c)| \leq \mathcal{L}(b, c)  \text{ for any } a,b,c \in \bbR^n. \tag*{\textbf{A1}}\\
\phantom{44}&\text{The function } \beta \mapsto P_{\lambda}(\beta) \text{ is } \mu\text{-strongly convex}. \label{as:2}\tag*{\textbf{A2}}
\end{align}
\reqnomode
The assumption on the loss function is verified for norms (regression) and indicator functions (classification).
Indeed, if $\mathcal{L}(a, b) = \normin{a - b}$, \ref{as:1} corresponds to the triangle inequality.
For the $\ell_{0/1}$-loss $\mathcal{L}(a, b) = \frac{1}{n}\sum_{i=1}^{n} \1_{a_i b_i < 0}$, since for any real $s,u$ and $v$, $|\1_{us < 0} - \1_{uv < 0}| \leq \1_{sv < 0}$, one has
$$\Big|\frac{1}{n}\sum_{i=1}^{n} \1_{a_i b_i < 0} - \frac{1}{n}\sum_{i=1}^{n} \1_{a_i c_i < 0}\Big| \leq \frac{1}{n}\sum_{i=1}^{n} \1_{b_i c_i < 0} \enspace.$$

\begin{definition} Given a primal solution $\tbeta{\lambda}$ for parameter $\lambda$ and a primal point $\beta^{(\lambda_{t})}$ returned by an algorithm, we define the gap on the validation error between $\lambda$ and $\lambda_{t}$ as
\begin{equation}
\Delta E_v(\lambda_{t}, \lambda): = \big|E_v(\tbeta{\lambda}) - E_v(\beta^{(\lambda_{t})})\big| \enspace.
\end{equation}
\end{definition}
Suppose we have fixed a tolerance $\epsilon_v$ on the gap on validation error \ie $\Delta E_v(\lambda_{t}, \lambda) \leq \epsilon_v$.
Based on Assumption \ref{as:1}, if there is a region $\mathcal{R}_{\lambda}$ that contains the optimal solution $\tbeta{\lambda}$ at parameter $\lambda$, then we have
\begin{align*}
\Delta E_v(\lambda_{t}, \lambda) &\leq \mathcal{L}({X'} \tbeta{\lambda}, {X'}\beta^{(\lambda_{t})})\\
&\leq \max_{\beta \in\mathcal{R}_{\lambda}} \mathcal{L}({X'} \beta, {X'}\beta^{(\lambda_{t})}) \enspace.
\end{align*}
A simple strategy consists in choosing $\mathcal{R}_{\lambda}$ as a ball.
\begin{lemma}[Gap safe region \protect{\citet{Ndiaye_Fercoq_Gramfort_Salmon17}}]\label{lm:safe_region}
Under \ref{as:2}, any primal solution $\tbeta{\lambda}$ belongs to the Euclidean ball with center $\beta^{(\lambda_{t})}$ and radius
\begin{equation}\label{eq:gap_radius}
r_{t, \mu}(\lambda) = \sqrt{\frac{2}{\mu}\Gap_{\lambda}(\beta^{(\lambda_{t})}, \theta^{(\lambda_{t})})} \enspace.
\end{equation}
\end{lemma}
Such a safe ball leveraging duality gap has been proved useful to speed-up sparse optimization solvers.
The improve performance relies on the ability to identify the sparsity structure of the optimal solutions; approaches of this type are referred to as \emph{safe screening rules} as they provide \emph{safe} certificates for such structures \citep{ElGhaoui_Viallon_Rabbani12, Fercoq_Gramfort_Salmon15, Shibagaki_Karasuyama_Hatano_Takeuchi16, Ndiaye_Fercoq_Gramfort_Salmon17}.

Since the radius in \Cref{eq:gap_radius} depends explicitly on the duality gap, we can sequentially track a range of parameters for which the gap on the validation error remains below a prescribed tolerance by controlling the optimization error.

\begin{proposition}[Grid for prescribed validation error]\label{prop:Grid_for_a_prescribed_validation_error}
Under Assumptions \ref{as:1} and \ref{as:2}, let us define for $i \in [n']$ an index in the test set, $\xi_i = \left(\tfrac{{x'}_{\!\!i}^{\top}\beta^{(\lambda_{t})}}{\norm{{x'}_{\!\!i}}}\right)^2$ and
\begin{equation}
\epsilon_{v,\mu} =
\begin{cases}
\frac{\mu}{2} \times \left(\frac{\epsilon_{v}}{\normin{X'}}\right)^2, & \text{(regression)}\\
 \frac{\mu}{2} \times \xi_{(\lfloor n \epsilon_v \rfloor + 1)}, & \text{(classification)}
\end{cases}
\end{equation}
where $\xi_{(\lfloor n \epsilon_v \rfloor + 1)}$ is the $(\lfloor n \epsilon_v \rfloor + 1)$-th
smallest value of $\xi_i$'s. Given $(\beta^{(\lambda_{t})}, \theta^{(\lambda_{t})})$ such that $\Gap_t \leq \epsilon_{v,\mu}$, we have $\Delta E_v(\lambda_{t}, \lambda) \leq \epsilon_v$ for all
%
% \begin{equation*}
parameter $\lambda$ in the interval\\
$\lambda_{t} \times \left[1 - \rho_{t}^{\ell}(\epsilon_{v,\mu}),\, 1 + \rho_{t}^{r}(\epsilon_{v,\mu}) \right],$
% \end{equation*}
%
where $\rho_{t}^{\ell}(\epsilon_{v,\mu})$, $\rho_{t}^{r}(\epsilon_{v,\mu})$ are defined in \Cref{prop:grid_for_a_prescribed_precision}.
\end{proposition}

\begin{remark}[Stopping criterion for training]
For the current parameter $\lambda_{t}$, % we have
$\Delta E_v(\lambda_{t}, \lambda_{t})
% =  \big|E_v(\tbeta{\lambda_{t}}) - E_v(\beta^{(\lambda_{t})})\big|
 \leq \epsilon_v$ as soon as $\Gap_t \leq \epsilon_{v,\mu}$, which gives us a stopping criterion for optimizing on the training part $(X, y)$ relative to the desired accuracy $\epsilon_v$ on the validation data $(X', y')$.
This has the appealing property of relieving the practitioner from selecting the stopping criterion $\epsilon_{c}$ when optimizing on the training set.
\end{remark}
% \begin{remark}
% % [Complexity of the Validation Path]
% Following the relation between the validation and the approximation path highlighted by our analysis, we identify the complexity bound (in $O(1/\sqrt{\epsilon_{v,\mu}})$ for strongly convex loss function) of the number of grid point needed to guarantee an $\epsilon_v$ approximation of the best hyperparameter. Extension to uniform convexity follows the same lines of analysis.
% \end{remark}

\Cref{alg:adaptive_validation_path} outputs a discrete set of parameters $\Lambda(\epsilon_{v,\mu})$ such that $\{\beta^{(\lambda_{t})} \text{ for } \lambda_{t} \in \Lambda(\epsilon_{v,\mu})\}$ is an {$\epsilon_v$-path} for the validation error.
Thus, for any $\lambda \in [\lambdamin, \lambdamax]$, there exists $\lambda_{t} \in \Lambda(\epsilon_{v,\mu})$ such that
\begin{align}
E_v(\beta^{(\lambda_{t})}) - \epsilon_v \leq E_v(\tbeta{\lambda}) \enspace.
\end{align}
The following proposition is obtained by taking the minimum on both
sides of the inequality.
%Taking successively the minimum over all $\lambda_{t}$ in the grid on the left hand side, and the minimum over all $\lambda$ in the parameter range on the right hand side, we obtain

\begin{proposition}\label{prop:global_convergence}
Under Assumptions \ref{as:1} and \ref{as:2}, the set $\{\beta^{(\lambda_{t})} \text{ for } \lambda_{t} \in \Lambda(\epsilon_{v,\mu})\}$ is an $\epsilon_v$-path for the error and
\begin{align*}
\min_{\lambda_{t} \in \Lambda(\epsilon_{v,\mu})} E_v(\beta^{(\lambda_{t})}) - \min_{\lambda \in [\lambdamin, \lambdamax]}E_v(\tbeta{\lambda}) \leq \epsilon_v \enspace.
\end{align*}
\end{proposition}

%!TEX root = ../icml_supp.tex

%%%%%%%%%%%%%%%%%%%%%%%%%%%%%%%%%%%%%%%%%%%%%%%%%%%%%%%%%%%%%%%%%%%%%%%%%%%%%%%
%%%%%%%%%%%%%%%%%%%%%%%%%%%%%%%%%%%%%%%%%%%%%%%%%%%%%%%%%%%%%%%%%%%%%%%%%%%%%%%
\section{Numerical experiments}
\label{sec:numerical_experiments}
%%%%%%%%%%%%%%%%%%%%%%%%%%%%%%%%%%%%%%%%%%%%%%%%%%%%%%%%%%%%%%%%%%%%%%%%%%%%%%%
%%%%%%%%%%%%%%%%%%%%%%%%%%%%%%%%%%%%%%%%%%%%%%%%%%%%%%%%%%%%%%%%%%%%%%%%%%%%%%%

We illustrate our method on $\ell_1$-regularized least squares and logistic regression by comparing the computational times and number of grid points needed to compute an $\epsilon$-path for a given range $[\lambda_{\min}, \lambda_{\max}]$ for several strategies.

The "Default grid" is the one used by default in the packages \texttt{glmnet} \citep{Friedman_Hastie_Tibshirani10} and \texttt{sklearn} \citep{Pedregosa_etal11}.
It is defined as $\lambda_t = \lambdamax \times 10^{-\delta t/(T-1)}$ (here $\delta=3$).
The proposed grids are the adaptive unilateral/bilateral and uniform unilateral/bilateral grids that are defined in Propositions~\ref{prop:unilateral}, \ref{prop:bilateral} and \ref{prop:uniform}.

Thanks to \Cref{prop:precision_for_a_given_grid}, we measure the approximation path error $\epsilon$ of the default grid of size $T$ and report the times and numbers of grid points $T_{\epsilon}$ needed to achieve such a precision.
Our experiments were conducted on the \texttt{leukemia} dataset, available in \texttt{sklearn} and the climate dataset \texttt{NCEP/NCAR Reanalysis} \citep{Kalnay_Kanamitsu_Kistler_Collins_Deaven_Gandin_Iredell_Saha_White_Woollen_Others96}.
The optimization algorithms are the same for all the grid, hence we compare only the grid construction impact.
Results are reported in \Cref{fig:bench_approximation_path} for classification and regression problem.
Our approach leads to better guarantees for approximating the regularization path w.r.t. the default grid and often significant gain in computing time.

\Cref{fig:bench_validation_path} illustrates convergence for Elastic Net (Enet) \citep{Zou_Hastie05}, on synthetic data generated by \texttt{sklearn} as random regression problems \texttt{make\_regression} and \texttt{make\_sparse\_uncorrelated} \citep{Celeux_ElAnbari_Marin_Robert12}.
For a decreasing levels of validation error, we represent the $\lambda$ selected by our algorithm and its corresponding safe interval. Even when the validation curve is non smooth and non convex, the output of the safe grid search converges to the global minimum as stated in \Cref{prop:global_convergence}.

% if any room left, would be better to add a conclusion:
%%%%%%%%%%%%%%%%%%%%%%%%%%%%%%%%%%%%%%%%%%%%%%%%%%%%%%%%%%%%%%%%%%%%%%%%%%%%%%%
%%%%%%%%%%%%%%%%%%%%%%%%%%%%%%%%%%%%%%%%%%%%%%%%%%%%%%%%%%%%%%%%%%%%%%%%%%%%%%%
\section{Conclusion}
\label{sec:conclusion}
%%%%%%%%%%%%%%%%%%%%%%%%%%%%%%%%%%%%%%%%%%%%%%%%%%%%%%%%%%%%%%%%%%%%%%%%%%%%%%%
%%%%%%%%%%%%%%%%%%%%%%%%%%%%%%%%%%%%%%%%%%%%%%%%%%%%%%%%%%%%%%%%%%%%%%%%%%%%%%%
We have shown how to efficiently construct one dimensional grids of regularization parameters for convex risk minimization, and to get an automatic calibration, optimal in term of hold-out test error.
Future research could examine how to adapt our framework to address multi-dimensional parameter grids.
This case is all the more interesting that it naturally arises when addressing non-convex problems, \eg MCP or SCAD, with re-weighted {$\ell_1$}-minimization.
Approximation of a full path then requires to optimize up to precision $\epsilon_c$ at each step, even for non promising hyperparameter, which is time consuming. Combining our approach with safe elimination procedures could provide faster hyperparameter selection algorithms.

\section*{Acknowledgements}

This work was supported by the Chair Machine Learning for Big Data at T\'el\'ecom ParisTech. TL acknowledges the support of JSPS KAKENHI Grant number $17K12745$.

% \textbf{Do not} include acknowledgements in the initial version of
% the paper submitted for blind review.

% If a paper is accepted, the final camera-ready version can (and
% probably should) include acknowledgements. In this case, please
% place such acknowledgements in an unnumbered section at the
% end of the paper. Typically, this will include thanks to reviewers
% who gave useful comments, to colleagues who contributed to the ideas,
% and to funding agencies and corporate sponsors that provided financial


\bibliography{references_all}

\providecommand{\AC}{A.-C}\providecommand{\CA}{C.-A}\providecommand{\CH}{C.-H}\providecommand{\CJ}{C.-J}\providecommand{\JC}{J.-C}\providecommand{\JP}{J.-P}\providecommand{\JB}{J.-B}\providecommand{\JF}{J.-F}\providecommand{\JJ}{J.-J}\providecommand{\JM}{J.-M}\providecommand{\KW}{K.-W}\providecommand{\PL}{P.-L}\providecommand{\RE}{R.-E}\providecommand{\SJ}{S.-J}\providecommand{\XR}{X.-R}\providecommand{\WX}{W.-X}\providecommand{\PL}{P.-L}\providecommand{\YX}{Y.-X}
\begin{thebibliography}{38}
\providecommand{\natexlab}[1]{#1}
\providecommand{\url}[1]{\texttt{#1}}
\expandafter\ifx\csname urlstyle\endcsname\relax
  \providecommand{\doi}[1]{doi: #1}\else
  \providecommand{\doi}{doi: \begingroup \urlstyle{rm}\Url}\fi

\bibitem[Arlot \& Celisse(2010)Arlot and Celisse]{Arlot_Celisse10}
Arlot, S. and Celisse, A.
\newblock A survey of cross-validation procedures for model selection.
\newblock \emph{Statistics surveys}, 4:\penalty0 40--79, 2010.

\bibitem[Az{\'e} \& Penot(1995)Az{\'e} and Penot]{Aze_Penot_95}
Az{\'e}, D. and Penot, J.-P.
\newblock Uniformly convex and uniformly smooth convex functions.
\newblock In \emph{Annales de la facult{\'e} des sciences de Toulouse}, pp.\
  705--730. Universit{\'e} Paul Sabatier, 1995.

\bibitem[Bach et~al.(2012)Bach, Jenatton, Mairal, and
  Obozinski]{Bach_Jenatton_Mairal_Obozinski12}
Bach, F., Jenatton, R., Mairal, J., and Obozinski, G.
\newblock Convex optimization with sparsity-inducing norms.
\newblock \emph{Foundations and Trends in Machine Learning}, 4\penalty0
  (1):\penalty0 1--106, 2012.

\bibitem[Bauschke \& Combettes(2011)Bauschke and
  Combettes]{Bauschke_Combettes11}
Bauschke, H.~H. and Combettes, P.~L.
\newblock \emph{Convex analysis and monotone operator theory in {H}ilbert
  spaces}.
\newblock Springer, New York, 2011.
\newblock ISBN 978-1-4419-9466-0.

\bibitem[Bottou \& Bousquet(2008)Bottou and Bousquet]{Bottou_Bousquet08}
Bottou, L. and Bousquet, O.
\newblock The tradeoffs of large scale learning.
\newblock In \emph{NIPS}, pp.\  161--168, 2008.

\bibitem[Celeux et~al.(2012)Celeux, Anbari, Marin, and
  Robert]{Celeux_ElAnbari_Marin_Robert12}
Celeux, G., Anbari, M.~E., Marin, J.-M., and Robert, C.~P.
\newblock Regularization in regression: comparing bayesian and frequentist
  methods in a poorly informative situation.
\newblock \emph{Bayesian Analysis}, 7\penalty0 (2):\penalty0 477--502, 2012.

\bibitem[Clarkson(2010)]{Clarkson10}
Clarkson, K.~L.
\newblock Coresets, sparse greedy approximation, and the frank-wolfe algorithm.
\newblock \emph{ACM Transactions on Algorithms (TALG)}, 6\penalty0
  (4):\penalty0 63:1--63:30, 2010.

\bibitem[Efron et~al.(2004)Efron, Hastie, Johnstone, and
  Tibshirani]{Efron_Hastie_Johnstone_Tibshirani04}
Efron, B., Hastie, T., Johnstone, I.~M., and Tibshirani, R.
\newblock Least angle regression.
\newblock \emph{Ann. Statist.}, 32\penalty0 (2):\penalty0 407--499, 2004.
\newblock With discussion, and a rejoinder by the authors.

\bibitem[{El Ghaoui} et~al.(2012){El Ghaoui}, Viallon, and
  Rabbani]{ElGhaoui_Viallon_Rabbani12}
{El Ghaoui}, L., Viallon, V., and Rabbani, T.
\newblock Safe feature elimination in sparse supervised learning.
\newblock \emph{J. Pacific Optim.}, 8\penalty0 (4):\penalty0 667--698, 2012.

\bibitem[Fercoq et~al.(2015)Fercoq, Gramfort, and
  Salmon]{Fercoq_Gramfort_Salmon15}
Fercoq, O., Gramfort, A., and Salmon, J.
\newblock Mind the duality gap: safer rules for the lasso.
\newblock In \emph{ICML}, pp.\  333--342, 2015.

\bibitem[Franceschi et~al.(2018)Franceschi, Frasconi, Salzo, and
  Pontil]{Franceschi_Frasconi_Salzo_Pontil18}
Franceschi, L., Frasconi, P., Salzo, S., and Pontil, M.
\newblock Bilevel programming for hyperparameter optimization and
  meta-learning.
\newblock In \emph{ICML}, pp.\  1563--1572, 2018.

\bibitem[Friedman et~al.(2007)Friedman, Hastie, H{\"o}fling, and
  Tibshirani]{Friedman_Hastie_Hofling_Tibshirani07}
Friedman, J., Hastie, T., H{\"o}fling, H., and Tibshirani, R.
\newblock Pathwise coordinate optimization.
\newblock \emph{Ann. Appl. Stat.}, 1\penalty0 (2):\penalty0 302--332, 2007.

\bibitem[Friedman et~al.(2010)Friedman, Hastie, and
  Tibshirani]{Friedman_Hastie_Tibshirani10}
Friedman, J., Hastie, T., and Tibshirani, R.
\newblock Regularization paths for generalized linear models via coordinate
  descent.
\newblock \emph{J. Stat. Softw.}, 33\penalty0 (1):\penalty0 1, 2010.

\bibitem[G{\"a}rtner et~al.(2012)G{\"a}rtner, Jaggi, and
  Maria]{Gartner_Jaggi_Maria12}
G{\"a}rtner, B., Jaggi, M., and Maria, C.
\newblock An exponential lower bound on the complexity of regularization paths.
\newblock \emph{Journal of Computational Geometry}, 2012.

\bibitem[Giesen et~al.(2010)Giesen, Jaggi, and Laue]{Giesen_Jaggi_Laue10}
Giesen, J., Jaggi, M., and Laue, S.
\newblock Approximating parameterized convex optimization problems.
\newblock In \emph{European Symposium on Algorithms}, pp.\  524--535, 2010.

\bibitem[Giesen et~al.(2012)Giesen, M{\"u}ller, Laue, and
  Swiercy]{Giesen_Laue_Mueller_Swiercy12}
Giesen, J., M{\"u}ller, J.~K., Laue, S., and Swiercy, S.
\newblock Approximating concavely parameterized optimization problems.
\newblock In \emph{NIPS}, pp.\  2105--2113, 2012.

\bibitem[Hastie et~al.(2004)Hastie, Rosset, Tibshirani, and
  Zhu]{Hastie_Rosset_Tibshirani_Zhu04}
Hastie, T., Rosset, S., Tibshirani, R., and Zhu, J.
\newblock The entire regularization path for the support vector machine.
\newblock \emph{J. Mach. Learn. Res.}, 5:\penalty0 1391--1415, 2004.

\bibitem[Hoerl \& Kennard(1970)Hoerl and Kennard]{Hoerl_Kennard70}
Hoerl, A.~E. and Kennard, R.~W.
\newblock Ridge regression: Biased estimation for nonorthogonal problems.
\newblock \emph{Technometrics}, 12\penalty0 (1):\penalty0 55--67, 1970.

\bibitem[Juditski \& Nesterov(2014)Juditski and Nesterov]{Juditski_Nesterov14}
Juditski, A. and Nesterov, Y.
\newblock Primal-dual subgradient methods for minimizing uniformly convex
  functions.
\newblock \emph{arXiv preprint arXiv:1401.1792}, 2014.

\bibitem[Kalnay et~al.(1996)Kalnay, Kanamitsu, Kistler, Collins, Deaven,
  Gandin, Iredell, Saha, White, Woollen,
  et~al.]{Kalnay_Kanamitsu_Kistler_Collins_Deaven_Gandin_Iredell_Saha_White_Woollen_Others96}
Kalnay, E., Kanamitsu, M., Kistler, R., Collins, W., Deaven, D., Gandin, L.,
  Iredell, M., Saha, S., White, G., Woollen, J., et~al.
\newblock The {NCEP/NCAR} 40-year reanalysis project.
\newblock \emph{Bulletin of the American meteorological Society}, 77\penalty0
  (3):\penalty0 437--471, 1996.

\bibitem[Li et~al.(2017)Li, Jamieson, DeSalvo, Rostamizadeh, and
  Talwalkar]{Li_Jamieson_DeSalvo_Rostamizadeh_Talwakar17}
Li, L., Jamieson, K., DeSalvo, G., Rostamizadeh, A., and Talwalkar, A.
\newblock Hyperband: A novel bandit-based approach to hyperparameter
  optimization.
\newblock \emph{The Journal of Machine Learning Research}, 18\penalty0
  (1):\penalty0 6765--6816, 2017.

\bibitem[Li \& Singer(2018)Li and Singer]{Li_Singer18}
Li, Y. and Singer, Y.
\newblock The well tempered lasso.
\newblock \emph{ICML}, 2018.

\bibitem[Mairal \& Yu(2012)Mairal and Yu]{Mairal_Yu12}
Mairal, J. and Yu, B.
\newblock Complexity analysis of the lasso regularization path.
\newblock In \emph{ICML}, pp.\  353--360, 2012.

\bibitem[McCullagh \& Nelder(1989)McCullagh and Nelder]{McCullagh_Nelder89}
McCullagh, P. and Nelder, J.~A.
\newblock \emph{Generalized Linear Models}.
\newblock Chapman \& Hall, 1989.

\bibitem[Ndiaye(2018)]{Ndiaye18}
Ndiaye, E.
\newblock \emph{{Safe optimization algorithms for variable selection and
  hyperparameter tuning}}.
\newblock PhD thesis, {Universit{\'e} Paris-Saclay}, 2018.

\bibitem[Ndiaye et~al.(2017)Ndiaye, Fercoq, Gramfort, and
  Salmon]{Ndiaye_Fercoq_Gramfort_Salmon17}
Ndiaye, E., Fercoq, O., Gramfort, A., and Salmon, J.
\newblock Gap safe screening rules for sparsity enforcing penalties.
\newblock \emph{J. Mach. Learn. Res.}, 18\penalty0 (128):\penalty0 1--33, 2017.

\bibitem[Nesterov(2004)]{Nesterov04}
Nesterov, Y.
\newblock \emph{Introductory lectures on convex optimization}, volume~87 of
  \emph{Applied Optimization}.
\newblock Kluwer Academic Publishers, Boston, MA, 2004.

\bibitem[Osborne et~al.(2000)Osborne, Presnell, and
  Turlach]{Osborne_Presnnell_Turlach00}
Osborne, M.~R., Presnell, B., and Turlach, B.~A.
\newblock A new approach to variable selection in least squares problems.
\newblock \emph{IMA J. Numer. Anal.}, 20\penalty0 (3):\penalty0 389--403, 2000.

\bibitem[Park \& Hastie(2007)Park and Hastie]{Park_Hastie07}
Park, M.~Y. and Hastie, T.
\newblock L1-regularization path algorithm for generalized linear models.
\newblock \emph{J. Roy. Statist. Soc. Ser. B}, 69\penalty0 (4):\penalty0
  659--677, 2007.

\bibitem[Pedregosa(2016)]{Pedregosa16}
Pedregosa, F.
\newblock Hyperparameter optimization with approximate gradient.
\newblock In \emph{ICML}, pp.\  737--746, 2016.

\bibitem[Pedregosa et~al.(2011)Pedregosa, Varoquaux, Gramfort, Michel, Thirion,
  Grisel, Blondel, Prettenhofer, Weiss, Dubourg, Vanderplas, Passos,
  Cournapeau, Brucher, Perrot, and Duchesnay]{Pedregosa_etal11}
Pedregosa, F., Varoquaux, G., Gramfort, A., Michel, V., Thirion, B., Grisel,
  O., Blondel, M., Prettenhofer, P., Weiss, R., Dubourg, V., Vanderplas, J.,
  Passos, A., Cournapeau, D., Brucher, M., Perrot, M., and Duchesnay, E.
\newblock Scikit-learn: Machine learning in {P}ython.
\newblock \emph{J. Mach. Learn. Res.}, 12:\penalty0 2825--2830, 2011.

\bibitem[Rockafellar(1997)]{Rockafellar97}
Rockafellar, R.~T.
\newblock \emph{Convex analysis}.
\newblock Princeton Landmarks in Mathematics. Princeton University Press,
  Princeton, NJ, 1997.
\newblock Reprint of the 1970 original, Princeton Paperbacks.

\bibitem[Rosset \& Zhu(2007)Rosset and Zhu]{Rosset_Zhu07}
Rosset, S. and Zhu, J.
\newblock Piecewise linear regularized solution paths.
\newblock \emph{Ann. Statist.}, 35\penalty0 (3):\penalty0 1012--1030, 2007.

\bibitem[Shibagaki et~al.(2015)Shibagaki, Suzuki, Karasuyama, and
  Takeuchi]{Shibagaki_Suzuki_Karasuyama_Takeuchi15}
Shibagaki, A., Suzuki, Y., Karasuyama, M., and Takeuchi, I.
\newblock Regularization path of cross-validation error lower bounds.
\newblock In \emph{NIPS}, pp.\  1666--1674, 2015.

\bibitem[Shibagaki et~al.(2016)Shibagaki, Karasuyama, Hatano, and
  Takeuchi]{Shibagaki_Karasuyama_Hatano_Takeuchi16}
Shibagaki, A., Karasuyama, M., Hatano, K., and Takeuchi, I.
\newblock Simultaneous safe screening of features and samples in doubly sparse
  modeling.
\newblock In \emph{ICML}, pp.\  1577--1586, 2016.

\bibitem[Sun \& Tran-Dinh(2017)Sun and Tran-Dinh]{Sun_Tran-Dinh17}
Sun, T. and Tran-Dinh, Q.
\newblock Generalized self-concordant functions: A recipe for newton-type
  methods.
\newblock \emph{Mathematical Programming}, 2017.

\bibitem[Tibshirani(1996)]{Tibshirani96}
Tibshirani, R.
\newblock Regression shrinkage and selection via the lasso.
\newblock \emph{J. Roy. Statist. Soc. Ser. B}, 58\penalty0 (1):\penalty0
  267--288, 1996.

\bibitem[Zou \& Hastie(2005)Zou and Hastie]{Zou_Hastie05}
Zou, H. and Hastie, T.
\newblock Regularization and variable selection via the elastic net.
\newblock \emph{J. Roy. Statist. Soc. Ser. B}, 67\penalty0 (2):\penalty0
  301--320, 2005.

\end{thebibliography}
\bibliographystyle{icml2019}

% \appendix
% \section{Do \emph{not} have an appendix here}

% \textbf{\emph{Do not put content after the references.}}
% %
% Put anything that you might normally include after the references in a separate
% supplementary file.

% We recommend that you build supplementary material in a separate document.
% If you must create one PDF and cut it up, please be careful to use a tool that
% doesn't alter the margins, and that doesn't aggressively rewrite the PDF file.
% pdftk usually works fine. 

% \textbf{Please do not use Apple's preview to cut off supplementary material.} In
% previous years it has altered margins, and created headaches at the camera-ready
% stage. 
%%%%%%%%%%%%%%%%%%%%%%%%%%%%%%%%%%%%%%%%%%%%%%%%%%%%%%%%%%%%%%%%%%%%%%%%%%%%%%%
%%%%%%%%%%%%%%%%%%%%%%%%%%%%%%%%%%%%%%%%%%%%%%%%%%%%%%%%%%%%%%%%%%%%%%%%%%%%%%%

\newpage
\onecolumn
%!TEX root = ../icml_supp.tex

\newpage
%%%%%%%%%%%%%%%%%%%%%%%%%%%%%%%%%%%%%%%%%%%%%%%%%%%%%%%%%%%%%%%%%%%%%%%%%%%%%%%
%%%%%%%%%%%%%%%%%%%%%%%%%%%%%%%%%%%%%%%%%%%%%%%%%%%%%%%%%%%%%%%%%%%%%%%%%%%%%%%
\section{Appendix}
%%%%%%%%%%%%%%%%%%%%%%%%%%%%%%%%%%%%%%%%%%%%%%%%%%%%%%%%%%%%%%%%%%%%%%%%%%%%%%%
%%%%%%%%%%%%%%%%%%%%%%%%%%%%%%%%%%%%%%%%%%%%%%%%%%%%%%%%%%%%%%%%%%%%%%%%%%%%%%%

%%%%%%%%%%%%%%%%%%%%%%%%%%%%%%%%%%%%%%%%%%%%%%%%%%%%%%%%%%%%%%%%%%%%%%%%%%%%%%%
\subsection{Generalized self-concordant functions}
%%%%%%%%%%%%%%%%%%%%%%%%%%%%%%%%%%%%%%%%%%%%%%%%%%%%%%%%%%%%%%%%%%%%%%%%%%%%%%%

%XXX todo harmonize counters in appendix with main part

% \begin{definition}[\cite{Sun_Tran-Dinh17}]
% a $\mathcal{C}^3$-convex function $f$ is $(M_{f}, \nu)$-generalized self-concordant of order $\nu \geq 2$ and $M_f\geq 0$ if for any $x \in \dom f$ and $u, v \in \bbR^n$:
% $$\left| \langle \nabla^3 f(x)[v]u,u \rangle\right| \leq M_f \norm{u}_{x}^{2}\norm{v}_{x}^{\nu-2}\norm{v}_{2}^{3-\nu}.$$
% \end{definition}

\begin{proposition}[\citet{Sun_Tran-Dinh17}, Proposition 10]
\label{prop:generalized_concordant_taylor}
If $(M_{f}, \nu)$-generalized self concordant, then
\begin{equation}
w_{\nu}(-d_{\nu}(x, y))\norm{y-x}_{x}^{2} \leq f(y) - f(x) - \langle \nabla f(x), y - x\rangle \leq w_{\nu}(d_{\nu}(x, y))\norm{y-x}_{x}^{2} \enspace,
\end{equation}
where the right-hand side inequality holds if $d_{\nu}(x,y)<1$ for the case $\nu>2$ and where
\begin{equation}
d_{\nu}(x, y) :=
\begin{cases}
M_f \norm{y - x}_2 &\text{ if } \nu = 2,\\
\left(\frac{\nu}{2} - 1\right) M_f \norm{y- x}_{2}^{3-\nu} \norm{y- x}_{x}^{\nu-2} &\text{ if } \nu > 2,
\end{cases}
\end{equation}
and
\begin{equation}\label{eq:def_w_nu}
w_{\nu}(\tau) :=
\begin{cases}
\frac{e^{\tau} - \tau - 1}{\tau^2} &\text{ if } \nu = 2,\\
\frac{-\tau - \log(1-\tau)}{\tau^2} &\text{ if } \nu = 3,\\
\frac{(1-\tau)\log(1-\tau) + \tau}{\tau^2} &\text{ if } \nu = 4,\\
\left(\frac{\nu-2}{4-\nu}\right) \frac{1}{\tau}\left[\frac{\nu-2}{2(3-\nu)\tau}\left( (1-\tau)^{\frac{2(3-\nu)}{2-\nu}} -1 \right) - 1 \right] &\text{ otherwise.}
\end{cases}
\end{equation}
\end{proposition}

The dual of the logistic loss is Generalized Self-Concordant with $M_{f^*} = 1, \nu = 4$. Power loss function $f_i(z) = (y_i - z)^q$ for $q \in (1, 2)$, popular in robust regression, is covered with $M_{f} = \frac{2 - q}{\sqrt[(2-q)]{q(q-1)}}, \nu = \frac{2(3-q)}{2-q}$. We refer to \citep{Sun_Tran-Dinh17} for more details and examples. 

\begin{remark}
Note that the relation between $\mathcal{U}_{f^\star}$, $\mathcal{V}_{f^\star}$ and $\mathcal{U}_f$, $\mathcal{V}_f$ is not always explicit. Uniform convexity and smoothness do not always hold simultaneously for primal $f$ and dual functions $f^*$ and one needs to carefully consider the regularity of the functions used. In general, we have $\mathcal{V}_f = \mathcal{U}^{*}_{f*}$ for uniformly smooth \citep[Coroll. 2.7]{Aze_Penot_95} and for self concordant, $\nu + \nu^* = 6$ for $\nu^* \in (0, 6)$ \citep[Prop 6]{Sun_Tran-Dinh17}.
\end{remark}

\begin{figure}[ht]
\center
\includegraphics[width=0.7\columnwidth, keepaspectratio]{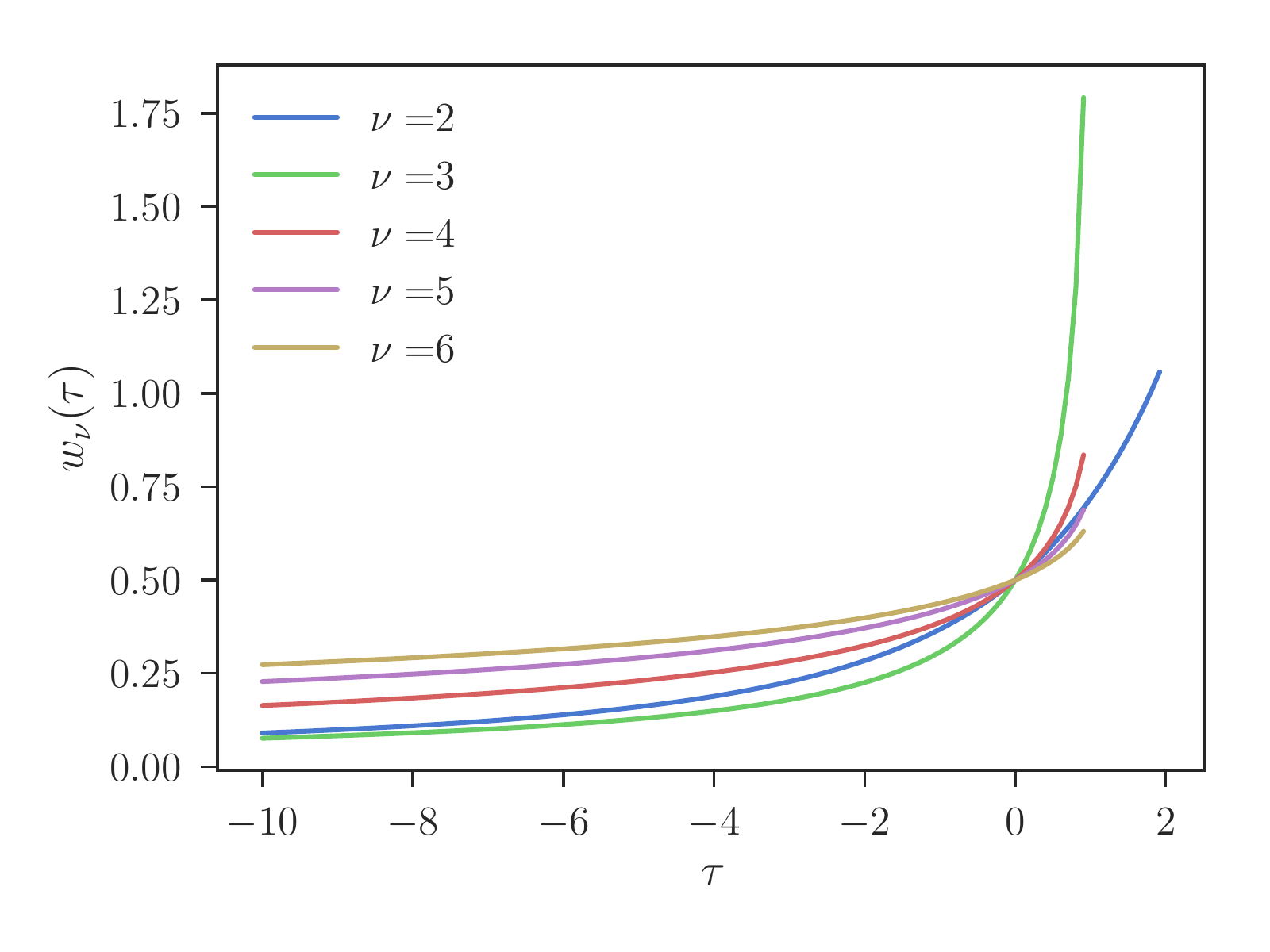}\label{fig:autoconco}
\caption{Illustration of the functions in self concordant bounds \Cref{eq:def_w_nu}}
\end{figure}

%%%%%%%%%%%%%%%%%%%%%%%%%%%%%%%%%%%%%%%%%%%%%%%%%%%%%%%%%%%%%%%%%%%%%%%%%%%%%%%
\subsection{Useful convexity inequalities}
%%%%%%%%%%%%%%%%%%%%%%%%%%%%%%%%%%%%%%%%%%%%%%%%%%%%%%%%%%%%%%%%%%%%%%%%%%%%%%%

\begin{lemma}[Fenchel-Young inequalities]\label{lm:fenchel_young}
Let $f$ be a continuously differentiable function. For all $x, x^*$, we have
\begin{align}
f(x) + f^*(x^*) &\geq \langle x^*, x \rangle \enspace, \label{eq:classic_fenchel_young}\end{align}
with equality if and only if $x^* = \nabla f(x)$ (or equivalently $x \in \partial f^*(x^*)$). Moreover, if
$f$ is $\mathcal{U}_{f,x}$-convex (resp. $\mathcal{V}_{f,x}$-smooth) 
 and $f^*$ is differentiable at $x^*$,
Inequality~\eqref{eq:mu_convex_fenchel_young} (resp. Inequality~\eqref{eq:phi_smooth_fenchel_young}) holds true:
\begin{align}
f(x) + f^*(x^*) &\geq \langle x^*, x \rangle + \mathcal{U}_{f,x}(x - \nabla f^*(x^*))\enspace, \label{eq:mu_convex_fenchel_young}\\
f(x) + f^*(x^*) &\leq \langle x^*, x \rangle + \mathcal{V}_{f,x}(x - \nabla f^*(x^*))\enspace. \label{eq:phi_smooth_fenchel_young}
\end{align}
\end{lemma}

\begin{proof}
We have from the $\mathcal{U}_{f,x}$-convexity and the equality $f(z) + f^*(\nabla f(z)) = \langle \nabla f(z), z\rangle$
% \begin{equation*}
% -f^*(\nabla f(z)) + \langle \nabla f(z), x \rangle + \frac{\mu}{2}\normin{x - z}_{2}^{2} = f(z) + \langle \nabla f(z), x-z \rangle + \frac{\mu}{2}\normin{x - z}_{2}^{2} \leq f(x).
% \end{equation*}
\begin{equation*}
-f^*(\nabla f(z)) + \langle \nabla f(z), x \rangle +
% \mu(\normin{x - z})
\mathcal{U}_{f,x}(x - z)
= f(z) + \langle \nabla f(z), x-z \rangle +
% \mu(\normin{x - z}) \leq f(x).
\mathcal{U}_{f,x}(x - z) \leq f(x)\enspace.
\end{equation*}
We conclude by applying the inequality at $z = \nabla f^*(x^*)$ and remark that $\nabla f(z) = x^*$.
% \begin{align*}
%  \langle x^*, x \rangle + \frac{\mu}{2}\normin{x - \nabla f^*(x^*)}_{2}^{2} \leq f(x) + f^*(x^*).
% \end{align*}
% \begin{align*}
%  \langle x^*, x \rangle + \mu(\normin{x - \nabla f^*(x^*)})\leq f(x) + f^*(x^*).
% \end{align*}
The same proof holds for the upper bound \eqref{eq:phi_smooth_fenchel_young}.
\end{proof}

Applying Fenchel-Young Inequalities~\eqref{eq:mu_convex_fenchel_young} and \eqref{eq:phi_smooth_fenchel_young} give the following bounds.
\begin{lemma}\label{lm:lower_bound_gap}
We assume that $-\lambda \theta \in \mathrm{Dom}(f^*)$ and $X^{\top} \theta \in \mathrm{Dom}(\Omega^{*})$. Then, the Inequality~\eqref{eq:lower_bound_gap} (resp. \eqref{eq:upper_bound_gap}) hold provided that $f$ is $\mathcal{U}_{f}$-convex (resp. $\mathcal{V}_{f}$-smooth).
\begin{align}
\lambda \widetilde \Omega(\beta, \theta) + \mathcal{U}_{f}(X\beta - \nabla f^*(-\lambda \theta)) &\leq \Gap_{\lambda}(\beta, \theta) \label{eq:lower_bound_gap}\\
\lambda \widetilde \Omega(\beta, \theta) + \mathcal{V}_{f}(X\beta - \nabla f^*(-\lambda \theta)) \enspace,
&\geq \Gap_{\lambda}(\beta, \theta)\enspace, \label{eq:upper_bound_gap}
\end{align}
where $\widetilde \Omega(\beta, \theta) = \Omega(\beta) + \Omega^{*}(X^\top \theta) + \langle \beta, -X^\top\theta \rangle$.
\end{lemma}
\begin{proof}
We apply the Fenchel-Young Inequality~\eqref{eq:mu_convex_fenchel_young} to obtain
\begin{align*}
\Gap_{\lambda}(\beta, \theta)
&= f(X\beta) + f^*(-\lambda \theta) + \lambda (\Omega(\beta) + \Omega^{*}(X^\top \theta))\\
&\geq  \langle X\beta, -\lambda \theta \rangle + \mathcal{U}_{f}(X\beta - \nabla f^*(-\lambda \theta)) + \lambda (\Omega(\beta) + \Omega^{*}(X^\top \theta))\\
&= \mathcal{U}_{f}(X\beta - \nabla f^*(-\lambda \theta)) + \lambda \left(\Omega(\beta) + \Omega^{*}(X^\top \theta) + \langle \beta, -X^\top\theta \rangle\right).
\end{align*}
The same technique applies for the upper bound with the Fenchel-Young Inequality \eqref{eq:phi_smooth_fenchel_young}.
\end{proof}
\begin{remark}
From the Fenchel-Young Inequality~\eqref{eq:classic_fenchel_young}, we have
$\Omega(\beta) + \Omega^{*}(X^\top \theta) \geq \langle \beta, X^\top\theta \rangle$, so the lower bound is always non negative.
\end{remark}

%%%%%%%%%%%%%%%%%%%%%%%%%%%%%%%%%%%%%%%%%%%%%%%%%%%%%%%%%%%%%%%%%%%%%%%%%%%%%%%
\subsection{Proof of the bounds for the approximation path error}
%%%%%%%%%%%%%%%%%%%%%%%%%%%%%%%%%%%%%%%%%%%%%%%%%%%%%%%%%%%%%%%%%%%%%%%%%%%%%%%

\begin{replemma}{lm:tracking_gap_regularization_parameter}
We assume that $-\lambda \theta^{(\lambda_{t})} \in \dom f^*$ and $X^{\top} \theta^{(\lambda_{t})} \in \dom\Omega^{*}$. If $f^*$ is $\mathcal{V}_{f^*}$-smooth (resp. $\mathcal{U}_{f^*}$-convex), then, for $\rho = 1-{\lambda}/{\lambda_{t}}$, the right (resp. left) hand side of Inequality~\eqref{eq:appendix_bounding_warm_start_gap} holds true
\begin{equation}
Q_{t, \mathcal{U}_{f^*}}(\rho)
\leq \Gap_{\lambda}(\beta^{(\lambda_{t})}, \theta^{(\lambda_{t})})
\leq Q_{t, \mathcal{V}_{f^*}}(\rho)\enspace. \label{eq:appendix_bounding_warm_start_gap}
\end{equation}
\end{replemma}

\begin{proof}
We recall that $\Gap_t := \Gap_{\lambda_t}(\beta^{(\lambda_{t})}, \theta^{(\lambda_{t})})$ and we denote for simplicity
\begin{align*}
	\Gap_{\lambda}^{\lambda_{t}} := \Gap_{\lambda}(\beta^{(\lambda_{t})}, \theta^{(\lambda_{t})}) \quad \text{ and } \quad \Gamma_t := \Omega(\beta^{(\lambda_{t})}) + \Omega^{*}(X^\top \theta^{(\lambda_{t})}) \enspace.
\end{align*}
By definition for any $(\beta, \theta) \in \dom P_{\lambda} \times \dom D_{\lambda} $ we have
$\Gap_{\lambda}(\beta, \theta) %&= P_{\lambda}(\beta) - D_{\lambda}(\theta)\\
= f(X\beta) + f^*(-\lambda \theta) + \lambda (\Omega(\beta) + \Omega^{*}(X^\top \theta))$,
so the following holds
\begin{align}
\frac{1}{\lambda_{t}}[\Gap_{t} - f(X\beta^{(\lambda_{t})}) - f^*(-\lambda_{t} \theta^{(\lambda_{t})})] = \Gamma_t \enspace. \label{eq:isolate_lambda_t}
\end{align}
Hence using Equality~\eqref{eq:isolate_lambda_t} in the definition of $\Gap_{\lambda}^{\lambda_{t}}$, we have:
\begin{align*}
\Gap_{\lambda}^{\lambda_{t}}  \overset{\phantom{\eqref{eq:isolate_lambda_t}}}{=} &f(X\beta^{(\lambda_{t})}) + f^*(-\lambda \theta^{(\lambda_{t})}) + \lambda \Gamma_t \\
\overset{\eqref{eq:isolate_lambda_t}}{=}&
\frac{\lambda}{\lambda_{t}} \Gap_{t} + \left(1 - \frac{\lambda}{\lambda_{t}}\right)[f(X\beta^{(\lambda_{t})}) + f^*(-\lambda_{t} \theta^{(\lambda_{t})})] + f^*(-\lambda \theta^{(\lambda_{t})}) - f^*(-\lambda_{t} \theta^{(\lambda_{t})})\enspace.
\end{align*}
Let us write the proof for the upper bound (the proof for the lower bound is similar).
We apply the smoothness property and the Fenchel-Young Inequality~\eqref{eq:phi_smooth_fenchel_young} to the function $f^*(\cdot)$ with $z = -\lambda \theta^{(\lambda_{t})}$ and $x = \zeta_{t} := -\lambda_{t} \theta^{(\lambda_{t})}$ to obtain
\begin{align*}
\Gap_{\lambda}^{\lambda_{t}} & \leq \frac{\lambda}{\lambda_{t}} \Gap_{t} +
\left(1 - \frac{\lambda}{\lambda_{t}}\right) \Delta_t +
\mathcal V_{f^*, \zeta_{t}}\left((\lambda_{t}-\lambda)\theta^{(\lambda_{t})}\right)\enspace,
\end{align*}
%
% \jo{wo why not have $\frac{\lambda}{\lambda_{t}} \Gap_{t}$ in the theorem instead. We need the assumption nowhere explicit that $\frac{\lambda}{\lambda_{t}}\leq 1$ to get the lemma correct. please double check.  }
where we have used the equality case in the Fenchel-Young Inequality~\eqref{eq:classic_fenchel_young} to get:
% \jo{more details are needed to explain why the inner product disappears}
\begin{align*}\Delta_t &= f(X\beta^{(\lambda_{t})}) + f^*(\zeta_{t}) + \langle \nabla f^*(\zeta_{t}) , - \zeta_t \rangle = f(X\beta^{(\lambda_{t})}) - f(\nabla f^*(\zeta_{t})) \enspace.
\end{align*}

We conclude by noticing that
$\frac{\lambda}{\lambda_{t}} \Gap_{t} + \left(1 - \frac{\lambda}{\lambda_{t}}\right) \Delta_t =
\Gap_{t} + \left(1 - \frac{\lambda}{\lambda_{t}}\right) (\Delta_t - \Gap_{t})
$, that $\zeta_{t}=-\lambda_{t}\theta^{(\lambda_{t})}$ and thanks to the definition of $Q_{t, \phi}$, from \Cref{eq:defined_phi}, applied to $\phi=\mathcal{V}_{f^*}$.
\end{proof}

\begin{repproposition}{prop:grid_for_a_prescribed_precision}[Grid for a prescribed precision]
Given $(\beta^{(\lambda_{t})}, \theta^{(\lambda_{t})})$ such that $\Gap_t \leq \epsilon_c < \epsilon$, for all
$\lambda \in \lambda_{t} \times \left[1 - \rho_{t}^{\ell}(\epsilon),\, 1 + \rho_{t}^{r}(\epsilon)\right]$,
we have $\Gap_{\lambda}(\beta^{(\lambda_{t})}, \theta^{(\lambda_{t})}) \leq \epsilon$ where $\rho_{t}^{\ell}(\epsilon)$ (resp. $\rho_{t}^{r}(\epsilon)$) is the largest non-negative $\rho$ s.t. $Q_{t, \mathcal{V}_{f^*}}(\rho) \leq \epsilon$ (resp. $Q_{t, \mathcal{V}_{f^*}}(-\rho) \leq \epsilon$).
\end{repproposition}

\begin{proof}
From \Cref{lm:tracking_gap_regularization_parameter}, we have
$\Gap_{\lambda}(\beta^{(\lambda_{t})}, \theta^{(\lambda_{t})}) \leq Q_{t, \mathcal{V}_{f^*}}(\rho)  = Q_{t, \mathcal{V}_{f^*}}(1 - \lambda/\lambda_t)$. Then, $\Gap_{\lambda}(\beta^{(\lambda_{t})}, \theta^{(\lambda_{t})}) \leq \epsilon$ for
$$\lambda \in \left[\inf\{\lambda' :\, Q_{t, \mathcal{V}_{f^*}}(1 - \lambda'/\lambda_t) \leq \epsilon, \;  \lambda' \leq \lambda_t \}, \quad \sup\{\lambda' :\, Q_{t, \mathcal{V}_{f^*}}(1 - \lambda'/\lambda_t) \leq \epsilon, \;  \lambda' \geq \lambda_t \}\right] \enspace.$$
\end{proof}

\begin{repproposition}{prop:precision_for_a_given_grid}[Precision for a Given Grid]
Given a grid of parameter $\Lambda_{T}$, the set $\{\beta^{(\lambda)} : \lambda \in \Lambda_T\}$ is an $\epsilon_{\Lambda_{T}}$-path and $\epsilon_{\Lambda_{T}} \leq \max_{t \in [T]} Q_{t,\mathcal{V}_{f^*}}(1-\lambda_{t}^{\star}/\lambda_{t}) $ where for all $t \in [T-1]$, $\lambda_{t}^{\star}$ is the largest $\lambda \in [\lambda_{t + 1},\lambda_{t}]$ such that $Q_{t,\mathcal{V}_{f^*}}(1-\lambda/\lambda_{t}) \geq Q_{t+1,\mathcal{V}_{f^*}}(1-\lambda/\lambda_{t + 1})$.
\end{repproposition}

\begin{proof}
From the upper bound $\Gap_{\lambda}(\beta^{(\lambda_t)}, \theta^{(\lambda_t)}) \leq Q_{t, \mathcal{V}_{f^*}}(1 - \lambda/\lambda_t)$ for all $\lambda$ and $\lambda_t$, and since one can partition the parameter set as $[\lambdamin, \lambdamax] = \cup_{t \in [0:T-1]} [\lambda_{t+1}, \lambda_t]$, we have thanks to the Definition of $\epsilon_{\Lambda_{t}}$ in \Cref{eq:def_epsilon_lambda_T}:
\begin{align*}
\epsilon_{\Lambda_{t}} &\leq \max_{t \in [0:T-1]} \sup_{\lambda \in [\lambda_{t+1}, \lambda_t]} \min_{\lambda_t \in \Lambda_{T}} Q_{t, \mathcal{V}_{f^*}}(1 - \lambda/\lambda_t)\\
&\leq\max_{t \in [0:T-1]} \sup_{\lambda \in [\lambda_{t+1}, \lambda_t]} \min_{t' \in \{t+1, t\}} Q_{t', \mathcal{V}_{f^*}}(1 - \lambda/\lambda_{t'}) \enspace.
\end{align*}
where the last inequality holds since $\{\lambda_{t+1}, \lambda_t\}$ is a subset of $\Lambda_{T}$.
Let us define
\begin{equation*}
\forall \lambda \in [\lambda_{t+1}, \lambda_t], \quad
\psi_{t}(\lambda) := \min\{Q_{t+1, \mathcal{V}_{f^*}}(1-\lambda/\lambda_{t+1}), Q_{t, \mathcal{V}_{f^*}}(1 - \lambda/\lambda_{t})\} \enspace.
\end{equation*}
The quantity $Q_{t+1, \mathcal{V}_{f^*}}(1 - \lambda/\lambda_{t+1})$ (resp. $Q_{t, \mathcal{V}_{f^*}}(1-\lambda/\lambda_{t})$) is monotonically increasing \wrt $\lambda$ (resp. decreasing), so $\sup_{\lambda \in [\lambda_{t+1}, \lambda_t]} \psi_{t}(\lambda)$ is reached for $\lambda_{t}^{\star}$, the largest $\lambda$ satisfying
$$Q_{t,\mathcal{V}_{f^*}}(1 - \lambda/\lambda_{t}) \geq Q_{t+1, \mathcal{V}_{f^*}}(1 - \lambda/ \lambda_{t + 1}) \enspace.$$
\end{proof}
% \jo{XXX adding a picture here would clarify the proof; to be done for rebutal at least}

\begin{repcorollary}{cor:quadratic_step_size}
If the function $f^*$ is $\frac{\nu}{2}\normin{\cdot}^2$-smooth, the left ($\rho^{\ell}_{t}$) and right ($\rho^{r}_{t}$) step sizes defined in \Cref{prop:grid_for_a_prescribed_precision} have closed-form expressions:
\begin{align*}
\rho^{\ell}_{t} = \frac{\sqrt{ 2 \nu \delta_{t} \norm{\zeta_{t}}^{2} + {\tilde{\delta}_{t}}^{2}} - \tilde{\delta}_{t}}{\nu \norm{\zeta_{t}}^{2}}
\text{ and }
\rho^{r}_{t} &= \frac{ \sqrt{  2 \nu\delta_{t} \norm{\zeta_{t}}^{2} + {\tilde{\delta}_{t}}^{2}} + \tilde{\delta}_{t}}{\nu \norm{\zeta_{t}}^{2}},
\end{align*}
where $\delta_{t} := \epsilon - \Gap_{t}$ and $\tilde{\delta}_{t} := \Delta_{t} - \Gap_{t}$.
 % and $\nu=\nu_{f^*}$ is the smoothness constant of the dual loss $f^*$.
\end{repcorollary}

\begin{proof}
If $f^*$ is $\frac{\nu}{2}\norm{\cdot}^2$-smooth (which is equivalent to $f$ is $\frac{1}{2\nu}\norm{\cdot}^2$-strongly convex), we have from \Cref{lm:tracking_gap_regularization_parameter}
\begin{align*}
\Gap_{\lambda}(\beta^{(\lambda_{t})}, \theta^{(\lambda_{t})}) \leq Q_{t, \mathcal{V}_{f^*}}(\rho)
= \Gap_{t} + \rho (\Delta_t - \Gap_{t}) + \frac{\nu \rho^2}{2}\norm{\zeta_t}^2 \enspace.
\end{align*}
Hence we conclude by solving in $\rho$ the inequality $Q_{t, \mathcal{V}_{f^*}}(\rho) \leq \epsilon$.
\end{proof}

\begin{replemma}{lem:dual_point}
For any $\beta^{(\lambda_{t})} \in \bbR^p$, the vector
 \begin{equation*}
 \theta^{(\lambda_{t})} =\frac{-\nabla f (X\beta^{(\lambda_{t})})}{\max(\lambda_{t}, \sigma_{\dom \Omega^*}^{\circ} (X^\top \nabla f(X\beta^{(\lambda_{t})}))}\enspace,
\end{equation*}
is feasible: $-\lambda \theta^{(\lambda_{t})} \in \dom f^*$, $X^{\top} \theta^{(\lambda_{t})} \in \dom \Omega^{*}$.
\end{replemma}

\begin{proof}
The proof of this result and the convergence of the sequence of dual points is given in Proposition 11 and lemma 5 of \citep[Chapter 2]{Ndiaye18}.
\end{proof}

\paragraph{Variation of the loss function along the path}
\begin{lemma}\label{lm:overfitting}
Let $\beta^{(\lambda_{t})}$ (resp. $\beta^{(\lambda_{t'})}$) be an $\epsilon$-solution at parameter $\lambda_{t}$ (resp. $\lambda_{t'}$), then we have
\begin{align*}
\left(1 - \frac{\lambda_{t'}}{\lambda_{t}}\right) \left(f(X\beta^{(\lambda_{t'})}) - f(X\beta^{(\lambda_{t})})\right) &\leq \Gap_{t'} + \frac{\lambda_{t'}}{\lambda_{t}} \Gap_{t}\enspace.
%
% \text{ \eug{remove this line } } &\leq \left(1 + \frac{\lambda}{\lambda_{t}} \right) \epsilon.
\end{align*}
where $\Gap_{s} := \Gap_{\lambda_s}(\beta^{(\lambda_s)}, \theta^{(\lambda_s)})$ for $s \in \{t, t'\}$. Moreover, the mapping $\lambda \mapsto f(X\tbeta{\lambda})$ is non-increasing.
\end{lemma}
\begin{proof}
Denote $\epsilon = \mathcal G_t$ and $\epsilon' = \mathcal G_{t'}$.
Since $\beta^{(\lambda_{t'})}$ is an $\epsilon'$-solution and $\tbeta{\lambda_{t'}}$ is optimal at parameter $\lambda_{t'}$, we have:
$$f(X\beta^{(\lambda_{t'})}) + \lambda_{t'} \Omega(\beta^{(\lambda_{t'})}) - \epsilon' \leq f(X\tbeta{\lambda_{t'}}) + \lambda_{t'} \Omega(\tbeta{\lambda_{t'}}) \leq f(X\beta^{(\lambda_t)}) + \lambda_{t'} \Omega(\beta^{(\lambda_t)}) \enspace.$$
Moreover,
\begin{align*}
f(X\beta^{(\lambda_t)}) + \lambda_{t'} \Omega(\beta^{(\lambda_t)})
&= \frac{\lambda_{t'}}{\lambda_t} \left(f(X\beta^{(\lambda_t)}) + \lambda_t \Omega(\beta^{(\lambda_t)})\right) + \left(1 - \frac{\lambda_{t'}}{\lambda_t} \right) f(X\beta^{(\lambda_t)})\\
&\leq \frac{\lambda_{t'}}{\lambda_t} \left(f(X\tbeta{\lambda_t}) + \lambda_t \Omega(\tbeta{\lambda_t}) + \epsilon\right) + \left(1 - \frac{\lambda_{t'}}{\lambda_t} \right) f(X\beta^{(\lambda_t)})\\
&\leq \frac{\lambda_{t'}}{\lambda_t} \left(f(X\beta^{(\lambda)}) + \lambda_t \Omega(\beta^{(\lambda)}) + \epsilon\right) + \left(1 - \frac{\lambda_{t'}}{\lambda_t} \right) f(X\beta^{(\lambda_t)}) \enspace.
\end{align*}
The last inequality comes from the optimality of $\tbeta{\lambda_t}$ at parameter $\lambda_t$.
Hence,
$$f(X\beta^{(\lambda_{t'})}) + \lambda_{t'} \Omega(\beta^{(\lambda_{t'})}) - \epsilon' \leq \frac{\lambda_{t'}}{\lambda_t} \left(f(X\beta^{(\lambda_{t'})}) + \lambda_t \Omega(\beta^{(\lambda_{t'})}) + \epsilon\right) + \left(1 - \frac{\lambda_{t'}}{\lambda_t} \right) f(X\beta^{(\lambda_t)})\enspace.$$
At optimality, $\epsilon=0$ and we can deduce that $\left(1 - \frac{\lambda}{\lambda_t} \right) f(X\tbeta{\lambda}) \leq \left(1 - \frac{\lambda}{\lambda_t} \right) f(X\tbeta{\lambda_t})$, hence the second result.
\end{proof}

\paragraph{Bounding the gradient along the path} \quad\\

We can furthermore bound the norm of the gradient of the loss when the parameter $\lambda$ varies.

\begin{lemma} \label{lm:bound_of_gradient}
For $x \in \mathrm{Dom}(f)$, if $f$ is $\mathcal{V}_{f, x}$-smooth, then writing $\mathcal{V}_{f, x}^{*}=(\mathcal{V}_{f, x})^{*}$ for the Fenchel-Legendre transform, one has
$$\mathcal{V}_{f, x}^{*}(-\nabla f(x)) \leq f(x) - \inf_z f(z)\enspace.$$
\end{lemma}
\begin{proof}
From the smoothness of $f$, we have
\begin{align*}
\inf_z f(z) &\leq \inf_z \left(f(x) + \langle \nabla f(x), z-x \rangle + \mathcal{V}_{f, x}(z - x) \right) = f(x) - (\mathcal{V}_{f, x})^*(-\nabla f(x))\enspace. 
\end{align*}
\end{proof}
% XXX Clarify the implication, a bit too vague.
A direct application of \Cref{lm:bound_of_gradient} and \Cref{lm:overfitting} yields:
% XXX TODO: here the impact of \epsilon vs. \epsilon_c should be investigated. Now it is blurry...
\begin{lemma}\label{lm:bound_norm_residual} Assume that $f$ is uniformly smooth and let $\beta^{(\lambda_{t'})}$ (resp. $\beta^{(\lambda_{t})}$) be an $\epsilon$-solution at parameter $\lambda_{t'}$ (resp. $\lambda_{t}$). Then for $\delta_{\epsilon}(\lambda_{t'}, \lambda_{t}) := \frac{\lambda_{t}+\lambda_{t'}}{\lambda_{t} - \lambda_{t'}} ~ \epsilon $, we have
\begin{equation*}
\mathcal{V}_{f}^{*}(-\nabla f(X\beta^{(\lambda_{t'})})) \leq f(X\beta^{(\lambda_{t})}) + \delta_{\epsilon}(\lambda_{t'}, \lambda_{t}) \enspace.
\end{equation*}
At optimality $\epsilon=0$ and so $\delta_{\epsilon}(\lambda_{t'}, \lambda_{t})=0$ and we have
\begin{equation*}
\mathcal{V}_{f}^{*}(-\nabla f(X\tbeta{\lambda_{t'}})) \leq f(X\tbeta{\lambda_{t}})\enspace.
\end{equation*}
\end{lemma}

\begin{replemma}{lm:bounds_gradient_and_gap}
Assuming $f$ uniformly smooth yields $\normin{\nabla f(X\beta^{(\lambda_{t'})})}_* \leq \widetilde R_{t}$, where $\widetilde R_{t} := {\mathcal{V}_{f}^{*}}^{-1}\big(f(X\beta^{(\lambda_{t})}) + \frac{2\epsilon_{c}}{\rho_{t}^{\ell} (\epsilon)}\big)$.
If additionally $f$ is uniformly convex, this yields $\Delta_{t'} \leq \widetilde \Delta_{t}$, where $\widetilde \Delta_{t} := \widetilde R_{t} \times \mathcal{U}_{f}^{-1}(\epsilon_{c})$ as well as $\Gap_{\lambda}(\beta^{(\lambda_{t'})}, \theta^{(\lambda_{t'})}) \leq Q_{t',\mathcal V_{f^*}}(\rho) \leq \widetilde Q_{t, \mathcal V_{f^*}}(\rho)$, where
\begin{equation*}\label{bound_on_gradients}
\widetilde Q_{t, \mathcal V_{f^*}}(\rho) =  \epsilon_{c} + \rho \cdot (\widetilde \Delta_{t} - \epsilon_{c} ) + \mathcal{V}_{f^*}\left(\left|\rho\right|\cdot\widetilde R_{t}\right) \enspace.
\end{equation*}
\end{replemma}

\begin{proof}
If $f$ is uniformly smooth, from \Cref{lm:bound_norm_residual}, we have:
\begin{align*}
\mathcal{V}_{f}^{*}(-\nabla f(X\beta^{(\lambda_{t'})})) &\leq f(X\beta^{(\lambda_{t})}) + \delta_{\epsilon}(\lambda_{t'}, \lambda_{t})\\
\normin{\nabla f(X\beta^{(\lambda_{t'})})}_* &\leq {\mathcal{V}_{f}^{*}}^{-1}\left(f(X\beta^{(\lambda_{t})}) + \frac{2 \epsilon}{\rho_{t}^{\ell}(\epsilon)}\right)\enspace,
\end{align*}
where the first line follows from \Cref{lm:bound_norm_residual} and the second follows from the fact that for the function $\mathcal{V}_{f} = \mathcal{V} \circ \norm{\cdot}$, we have $\mathcal{V}_{f}^{*} := (\mathcal{V}_{f})^{*} = \mathcal{V}^* \circ \norm{\cdot}_*$. Finally, since $\lambda_{t'} \leq \lambda_t$, then
$$\delta_{\epsilon}(\lambda_{t'}, \lambda_{t}) \leq \frac{2\epsilon \lambda_t}{\lambda_t - \lambda_{t'}} = \frac{2\epsilon}{\rho_{t}^{\ell}(\epsilon)} \enspace.$$

Since $f$ is convex, we have
\begin{align*}
\Delta_{t} &:= f(X\beta^{(\lambda_{t})}) - f(\nabla f^*(-\lambda_{t} \theta^{(\lambda_{t})}))
\leq -\langle \nabla f(X\beta^{(\lambda_{t})}), \nabla f^*(-\lambda_{t} \theta^{(\lambda_{t})}) - X\beta^{(\lambda_{t})} \rangle\\
&\leq \normin{\nabla f(X\beta^{(\lambda_{t})})}_* \times \normin{\nabla f^*(-\lambda_{t} \theta^{(\lambda_{t})}) - X\beta^{(\lambda_{t})}}\\
&\leq \normin{\nabla f(X\beta^{(\lambda_{t})})}_* \times \mathcal{U}_{f}^{-1}(\Gap_{\lambda_{t}}(\beta^{(\lambda_{t})}, \theta^{(\lambda_{t})})) \enspace.
\end{align*}
where the two last inequalities come respectively from Holder inequality and \Cref{lm:lower_bound_gap}.

The bound on the duality gap directly comes from the bounds on $\Delta_{t'}$ and the norm of the gradient.
\end{proof}

%%%%%%%%%%%%%%%%%%%%%%%%%%%%%%%%%%%%%%%%%%%%%%%%%%%%%%%%%%%%%%%%%%%%%%%%%%%%%%%
\subsection{Proof of the complexity bound}
%%%%%%%%%%%%%%%%%%%%%%%%%%%%%%%%%%%%%%%%%%%%%%%%%%%%%%%%%%%%%%%%%%%%%%%%%%%%%%%

\begin{repproposition}{prop:uniform}
Assume that $f$ is uniformly convex and smooth, and define the grid $\Lambda^{(0)}(\epsilon) = \{\lambda_0, \ldots, \lambda_{T_{\epsilon}-1} \}$ by
$$\lambda_0 = \lambda_{\max},\quad  \lambda_{t + 1} = \lambda_{t} \times (1 - \rho_{0}(\epsilon))\enspace,$$
and $\forall t \in [T], (\beta^{(\lambda_{t})}, \theta^{(\lambda_{t})})$ s.t. $\Gap_t \leq \epsilon_{c} < \epsilon$.
Then the set $\{\beta^{(\lambda_{t})} \;:\; \lambda_{t} \in \Lambda^{(0)}(\epsilon)\}$ is an $\epsilon$-path for Problem~\eqref{eq:primal_problem}
 with at most $T_{\epsilon}$ grid points where
 \begin{equation}
 T_{\epsilon} = \left\lfloor \frac{\log({\lambdamin}/{\lambdamax})}{\log(1 - \rho_{0}(\epsilon))} \right\rfloor \enspace.
 \end{equation}
\end{repproposition}

\begin{proof}
By construction, given any two consecutive grid point $\lambda_t$ and $\lambda_{t+1}$, we have $\Gap_t$ and $\Gap_{t+1}$ are smaller than $\epsilon$. Moreover, since $\beta^{(\lambda_t)}$ is an $\epsilon$-solution for any $\lambda$ in the interval $[\lambda_{t+1}, \lambda_t]$ whose union forms a covering of $[\lambdamin, \lambdamax]$, we conclude that the uniform grid is an $\epsilon$-path.

By definition, $\lambdamin = \lambda_{T_{\epsilon}}$, $\lambdamax = \lambda_{0}$ and $\rho_0(\epsilon) = 1 - \lambda_{t+1}/\lambda_t \in (0,1)$. The conclusion follows from the fact that $\lambda_{T_{\epsilon}} = \lambda_0 \times (1 - \rho_{0}(\epsilon))^{T_{\epsilon}}$.
\end{proof}

Now we present the proof for the general case. We denote $T_{\epsilon}$ the cardinality of the grid returned by Algorithm~\ref{alg:adaptive_training_path} and let $(\rho_{t})_{t \in [0:T_{\epsilon}-1]}$ be the set of step size needed to cover the interval $[\lambdamin, \lambdamax]$. Using $\rho_{t} = 1 - \frac{\lambda_{t + 1}}{\lambda_{t}}$
, we have
\begin{equation*}
\log \left(\frac{\lambdamax}{\lambdamin}\right)
= \log\left(\prod_{t=0}^{T_{\epsilon}-1}\frac{\lambda_{t}}{\lambda_{t + 1}}\right)
= \sum_{t=0}^{T_{\epsilon}-1} \log\left(\frac{1}{1-\rho_{t}}\right) \enspace.
\end{equation*}

Hence, denoting $\rho_{\min}(\epsilon) = \min_{t \in [0:T_{\epsilon}-1]} \rho_{t}$, we have
\begin{equation}\label{eq:first_bound_complexity}
T_{\epsilon} \times \rho_{\min}(\epsilon) \leq \log \left(\frac{\lambdamax}{\lambdamin}\right) \enspace.
\end{equation}

We assume that we explore the parameter range in decreasing order.
Also, to simplify the complexity analysis we will suppose that at each step $\lambda_t$, we have solved the optimization problem with two measures of accuracy $\Gap_{t} \leq \epsilon_{c}$ and $\Delta_{t} \leq \epsilon_{c}$ for $\epsilon_{c}<\epsilon$.

\begin{remark}
It is important to note that the usual stop criterion at each step $t$ is $\Gap_{t} \leq \epsilon_{c}$ which is used in our algorithm. The additional constraint $\Delta_{t} \leq \epsilon_{c}$ can be satisfied by any converging optimization solver (\eg coordinate descent) since both $\Gap_{t}$ and $\Delta_{t}$ converge to zero. 
\end{remark}

Then we recall from \Cref{lm:tracking_gap_regularization_parameter} that
$\Gap_{\lambda}(\beta^{(\lambda_{t})}, \theta^{(\lambda_{t})}) \leq Q_{t,\mathcal{V}_{f^*}}(\rho)$ which is smaller than $\epsilon$ as soon as
$\mathcal{V}_{f^*, \zeta_t}(-\zeta_t \cdot \rho) \leq \epsilon - \epsilon_{c}$. Since
$\rho_{\min}(\epsilon) = \min_{t \in [0:T_{\epsilon}-1]} \rho_{t} = \min_{t \in [0:T_{\epsilon}-1]} \sup \{\rho \; : \; Q_{t,\mathcal{V}_{f^*}}(\rho) \leq \epsilon \}$, then
\begin{align}\label{eq:general_lower_bound_rho}
\rho_{\min}(\epsilon) \geq \min_{t \in [0:T_{\epsilon}-1]} \sup \{\rho \; : \; \mathcal{V}_{f^*, \zeta_t}(-\zeta_t \cdot \rho) \leq \epsilon - \epsilon_{c} \} \enspace.
\end{align}
Hence the complexity of the path is bounded as follows.

\begin{repproposition}{prop:approx_path_complexity}[Complexity of the approximation path]
Assuming that $\max(\Gap_t, \Delta_t) \leq \epsilon_{c} < \epsilon$ at each step $t$, there exists $C_{f}(\epsilon_{c})>0$ such that
 \begin{align*}
T_{\epsilon} & \leq \log\left(\frac{\lambdamax}{\lambdamin}\right) \times \frac{C_{f}(\epsilon_{c})}{\mathcal{W}_{f^*}(\epsilon - \epsilon_{c})} \enspace,
\end{align*}
where for all $t>0$, the function $\mathcal{W}_{f^*}$ is defined by
\begin{align*}
\mathcal{W}_{f^*} =
\begin{cases}
\mathcal{V}_{f^*}^{-1}, &  \text{if } f  \text{ is uniformly convex and smooth,}\\
\sqrt{\cdot}, &  \text{if } f \text{ is Generalized Self-Concordant and uniformly-smooth.}
\end{cases}
\end{align*}
Moreover,
\begin{align*}
C_{f}(\epsilon_{c}) =
\begin{cases}
\widetilde R_{0} , &  \text{if } f  \text{ is uniformly convex and smooth,}\\
\sqrt{ \frac{\bar R_0}{w_{\nu}(-B_f)}}, &  \text{if } f \text{ is Generalized Self-Concordant and uniformly-smooth.}
\end{cases}
\end{align*}
where 
\begin{equation*}
\widetilde R_{0} = {\mathcal{V}_{f}^{*}}^{-1}\left(f(X\beta^{(\lambda_{0})}) + \frac{2\epsilon_{c}}{\rho_{0}^{\ell} (\epsilon)}\right), \text{ and } \bar R_0 = f(X \beta^{(\lambda_0)}) + \frac{2\epsilon_{c}}{\rho_0^\ell(\epsilon)} + \epsilon_{c} 
\end{equation*}
and
\begin{equation*}
B_f = \sup\{ d_{\nu}(z) : z \in \bbR^n, \psi(z) \leq \bar R_0,  \norm{z}_* \leq \widetilde R_0 \} \enspace.
\end{equation*}
\end{repproposition}

\begin{proof}
In the uniformly convex case, $\mathcal{V}_{f^*, \zeta_t}(-\zeta_t \cdot \rho) = \mathcal{V}_{f^*}(\rho \norm{\zeta_t}_*)$, hence we can deduce from \Cref{eq:first_bound_complexity} and \eqref{eq:general_lower_bound_rho} that
\begin{align*}
T_{\epsilon} &\leq \frac{1}{\rho_{\min}(\epsilon)} \times \log\left(\frac{\lambdamax}{\lambdamin}\right)
\leq \log\left(\frac{\lambdamax}{\lambdamin}\right) \times \frac{\max_{t \in [0:N_\epsilon-1]} \norm{\zeta_t}_*}{\mathcal{V}_{f^*}^{-1}(\epsilon - \epsilon_{c})} \enspace,
\end{align*}
so we just need to uniformly bound $\norm{\zeta_s}_*$. By construction of the dual point \Cref{eq:def_dual_vector}, we have:
\begin{align}\label{eq:bound_norm_zeta_t}
\norm{\zeta_t}_* = \frac{\lambda_t}{\max(\lambda_{t}, \sigma_{\dom \Omega^*}^{\circ} (X^\top \nabla f(X\beta^{(\lambda_{t})}))} \normin{\nabla f(X\beta^{(\lambda_t)})}_* \leq \normin{\nabla f(X\beta^{(\lambda_t)})}_* \leq \widetilde R_{0} \enspace,
\end{align}
where the last inequality comes from \Cref{lm:bounds_gradient_and_gap} applied at $\lambda_t$ and $\lambda_0$.

% Moreover, using the smoothness of $f$, we have shown that $\mathcal{V}^*(-\nabla(f(x)) \leq f(x)$ which implies
% %
% % \begin{equation}
% % \normin{\nabla f(X\beta^{\lambda_t})}_{*}^{2} \leq 2 \nu_{f} \times f(X\beta^{\lambda_t}) \leq
% % 2 \nu_{f} \times \left(f(X\beta^{\lambda_0}) + \frac{2\epsilon_c}{\rho_{l}(\epsilon_c)}\right) = \widetilde R_0^{'2}.
% % \end{equation}
% %
% Note that since $f$ is (essentially) smooth, this bound implies that $X\beta^{\lambda_t}$ does not go to the boundary of the domain of $f$. We conclude that $\norm{\zeta_t}_* \leq \widetilde R_0$.

For the Generalized Self-Concordant case, we first recall that the functions $w_{\nu}(\cdot)$ in \Cref{eq:def_w_nu} are increasing and $w_{\nu}(0)= {1}/{2}$. Then there exists a positive constant $a_{\nu}$ such that $w_{\nu}(\tau) \leq 1$ for $ \tau \in [0, a_{\nu}]$ (in fact $a_{\nu}=1$ for the logistic regression). Thus, provided $\rho d_\nu(\zeta_t) \leq a_\nu$, we can derive the bound
$\mathcal{V}_{f^*}(-\zeta_t \times \rho) \leq \rho^2 \norm{\zeta_t}_{\zeta_t}^{2}$.

Like in the uniformly convex case, in order to get the complexity of the $\epsilon$-path, we also need a uniform bound on $\norm{\zeta_t}_{\zeta_t}$. By taking \eqref{eq:mu_convexity} on $f^*$ with $x = \zeta_{t}$ and $z = 0$, we obtain
\begin{align*}
w_{\nu}(- d_{\nu}(\zeta_t)) \norm{\zeta_t}_{\zeta_t}^{2} = \mathcal U_{f^*, \zeta_{t}}(-\zeta_{t}) &\leq f^*(0) - f^*(\zeta_{t}) - \langle \nabla f^*(\zeta_{t}), -\zeta_{t}\rangle = f(\nabla f^*(\zeta_{t})) = f(X \beta^{(\lambda_{t})}) - \Delta_{t} \\
& \leq f(X \beta^{(\lambda_{t})}) + \epsilon_{c} \leq f(X \beta^{(\lambda_0)}) + \frac{2\epsilon_{c}}{\rho_0^\ell(\epsilon)} + \epsilon_{c} =: \bar R_0\enspace,
\end{align*}
where we used the inequality case of Fenchel-Young Inequality and the fact that $f^*(0) = -\inf f = 0$. This shows that $\psi(\zeta_t) := \mathcal U_{f^*, \zeta_{t}}(-\zeta_{t}) \leq \bar R_0$. Since the function $\psi$ is continuous, then its level set is closed \ie  $\{z \in \bbR^n: \psi(z) \leq \bar R_0\}$ is closed. Recalling \Cref{eq:bound_norm_zeta_t}, we have $\norm{\zeta_t}_* \leq \widetilde R_0$. Then we have
\begin{equation*}
d_{\nu}(\zeta_t) \leq \sup_{z \in \mathcal{H}_f} d_{\nu}(z) =: B_f \text{ where } \mathcal{H}_f := \{z \in \bbR^n: \psi(z) \leq \bar R_0 \} \cap \{z: \norm{z}_* \leq \widetilde R_0 \} \text{ is a compact set.}
\end{equation*}
Since the function $w_{\nu}(\cdot)$ is increasing, we have $w_{\nu}(-B_f)\norm{\zeta_t}_{\zeta_t}^{2} \leq w_{\nu}(- d_{\nu}(\zeta_t)) \norm{\zeta_t}_{\zeta_t}^{2} \leq \bar R_0$.

This implies that $\norm{\zeta_t}_{\zeta_t}^{2} \leq \frac{\bar R_0}{w_{\nu}(-B_f)}$. Thus, provided $\rho d_\nu(\zeta_t) \leq \bar a_\nu$, we can derive the bound $\mathcal{V}_{f^*}(\zeta_t \times \rho) \leq \rho^2 \norm{\zeta_t}_{\zeta_t}^{2}$.

Whence $\rho_{\min}(\epsilon) \geq \min_{t} \frac{\sqrt{\epsilon - \epsilon_c}}{\norm{\zeta_t}_{\zeta_t}}$. Hence the complexity is bounded as
$T_{\epsilon} \leq \log\left(\frac{\lambdamax}{\lambdamin}\right) \frac{\sqrt{\bar R_0/w_{\nu}(-B_f)}}{\sqrt{\epsilon - \epsilon_{c}}}$.
\end{proof}

% We can easily show that $\frac{1}{\rho_{0}^{\ell}(\epsilon)}$ goes to $+\infty$ in $O(1/\epsilon) = O(1/\epsilon_c)$ \citep{Giesen_Jaggi_Laue10}.
% Note that the initial step size $\rho_{0}^{\ell}(\epsilon) \leq \mathcal{V}_{f^*}^{-1}(\epsilon - \epsilon_{c})/\norm{\zeta_0}$.
% % which implies that $\frac{1}{\rho_{0}^{\ell}(\epsilon)}$ goes to $+\infty$ in $O(1/\epsilon) = O(1/\epsilon_c)$.
% Hence for loss function $f$ such that $\mathcal{V}_{f^*}(\cdot) = \nu_{f^*}\norm{\cdot}^{d}/d$, the complexity of the $\epsilon$-path corresponds to $T_{\epsilon} \in O({1}/{\sqrt[d]{\epsilon}})$

\subsection{Proof of the validation error bounds}

\begin{repproposition}{prop:Grid_for_a_prescribed_validation_error}[Grid for a prescribed validation error]
Suppose that we have solved problem~\eqref{eq:primal_problem} for a parameter $\lambda_{t}$ up to accuracy $\Gap_{\lambda_{t}}(\beta^{(\lambda_{t})}, \theta^{(\lambda_{t})}) \leq \xi(\epsilon_v,\mu, {X'})$, then we have $\Delta E_v(\lambda_{t}, \lambda) \leq \epsilon_v$ for all $$\lambda \in \lambda_{t} \times \left[1 - \rho_{t}^{\ell}(\xi(\epsilon_v,\mu, {X'})),\, 1 + \rho_{t}^{r}(\xi(\epsilon_v,\mu, {X'})) \right]\enspace,$$
where $\rho_{t}^{\ell}(\epsilon)$ and $\rho_{t}^{r}(\epsilon)$ for $\epsilon>0$ are defined in \Cref{prop:grid_for_a_prescribed_precision}.
\end{repproposition}

\begin{proof}
We distinguish the two cases of interest: classification and regression.
\begin{itemize}
\item{Case where the loss function is a norm:}

we have
\begin{align*}
\max_{\beta \in \mathcal{B}(\beta^{(\lambda_{t})}, r)}  \mathcal{L}({X'} \beta,  {X'}\beta^{(\lambda_{t})}) = \max_{\beta \in \mathcal{B}(\beta^{(\lambda_{t})}, r)} \normin{{X'}(\beta - \beta^{(\lambda_{t})})} \leq r_{\lambda, \mu} \normin{{X'}}\enspace,
\end{align*}
where $r_{\lambda, \mu}$ is the duality gap safe radius defined in \Cref{eq:gap_radius}. Hence by using the bounds on the duality gap in \Cref{lm:tracking_gap_regularization_parameter}, we can ensure $\Delta E_v(\lambda_{t}, \lambda) \leq \epsilon_v$  for all $\rho= 1 - \lambda/\lambda_t$ such that
$Q_{t,\mathcal{V}_{f^*}}(\rho) \leq \frac{\mu\epsilon_{v}^{2}}{2\normin{{X'}}^2}.$

\item Case where the loss function is the indicator function:

using the inequality
$-2ab \leq (a - b)^2 - b^2$ for $a={x'}_{i}^{\top} \beta$ and $b={x'}_{i}^{\top}\beta^{(\lambda_t)}$ and $|{x'}_{i}^{\top}(\beta - \beta^{(\lambda_t)})| \leq r \normin{{x'}_{i}}$  for all $\beta \in \mathcal{B}(\beta^{(\lambda_t)}, r)$ we have:
$$-2({x'}_{i}^{\top} \beta) ({x'}_{i}^{\top}\beta^{(\lambda_t)}) \leq (r \normin{{x'}_{i}})^2 - ({x'}_{i}^{\top}\beta^{(\lambda_t)})^2\enspace.$$
Hence we obtain the following upper bound
\begin{align*}
\max_{\beta \in \mathcal{B}(\beta^{(\lambda_t)}, r)} \mathcal{L}({X'} \beta,  {X'}\beta^{(\lambda_t)})
&= \max_{\beta \in \mathcal{B}(\beta^{(\lambda_t)}, r)} \frac{1}{n}\sum_{i=1}^{n} \1_{({x'}_{i}^{\top}\beta^{(\lambda_t)}) ({x'}_{i}^{\top}\beta) < 0} \leq \frac{1}{n} \sum_{i=1}^{n} \1_{|{x'}_{i}^{\top}\beta^{(\lambda_t)}|\leq r \normin{{x'}_{i}}}\enspace.
\end{align*}

By using the bound on the duality gap, we can ensure $\Delta E_v(\lambda_0, \lambda) \leq \epsilon_v$  for all $\lambda$ such that:
\begin{equation*}
\#\left\{ i \in [n]: \xi_i := \frac{\mu}{2} \left(\frac{{x'}_{i}^{\top}\beta^{(\lambda_t)}}{\norm{x_i'}}\right)^2 \leq Q_{t, \mathcal{V}_{f^*}}(1 - \lambda/\lambda_t) \right\} \leq \lfloor n \epsilon_v \rfloor \enspace.
\end{equation*}
By denoting $\left(\xi_{(i)}\right)_{i \in [n]}$ the (increasing) ordered sequence, we need the inequality to be true for at most the $\lfloor n \epsilon_v \rfloor$ first values \ie we choose $\lambda$ such that:
\begin{equation*}
Q_{t, \mathcal{V}_{f^*}}\left(1 - \frac{\lambda}{\lambda_t} \right) < \frac{\mu}{2} \left(\frac{{x'}_{(\lfloor n \epsilon_v \rfloor + 1)}^{\top}\beta^{(\lambda_t)}}{\norm{{x'}_{(\lfloor n \epsilon_v \rfloor + 1)}}}\right)^2 \enspace.
\end{equation*}
\end{itemize}
\end{proof}
%XXX TODO? may be add a case where $Q_{t, \mathcal{V}_{f^*}}(1 - \lambda/\lambda_t)$ can be simplified and make the conditionning explicit.

\subsection{Additional Experiments}

We add an additional experiments in large scale data with $n=16087$ observations and $p=1668737$ features.

\begin{figure}[ht]
\centering
\includegraphics[width=0.8\linewidth, keepaspectratio]{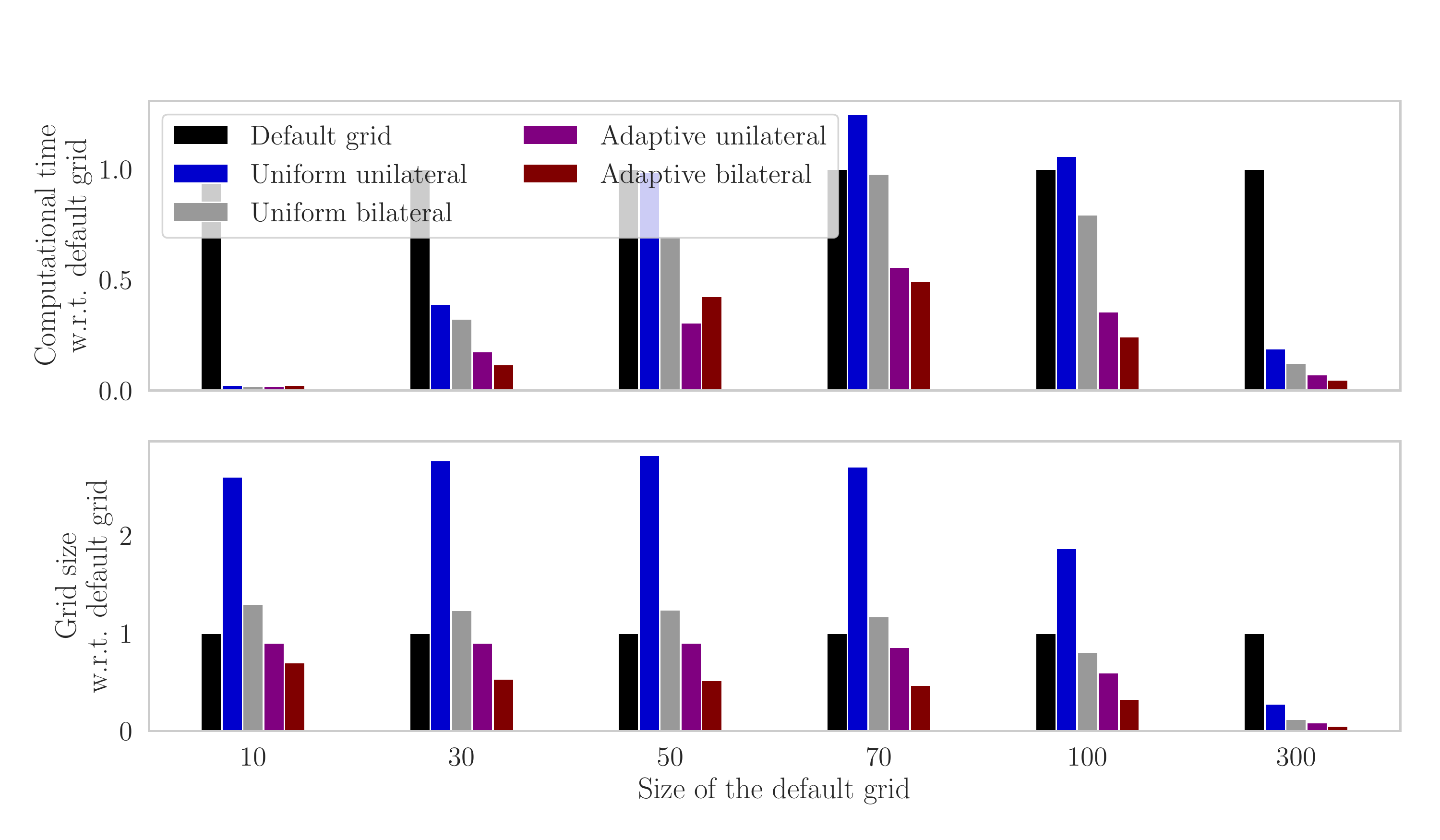}
\caption{
$\ell_1$ least-squares regression on the financial dataset \texttt{E2006-log1p} (available in libsvm) with $n=16087$ observations and $p=1668737$ features. We have used the same (vanilla) coordinate descent optimization solver with warm start between parameters for all grids. Note that a smaller grid do not imply faster computation, as the interplay with the warm-start can be intricate in our sequential approach.} \label{fig:bench_Lasso_finance_grid_n_lambdas_tau10}
\end{figure}

\end{document}